\documentclass[10pt,letter]{article}
\usepackage[utf8x]{inputenc}
\usepackage[T1]{fontenc}
\usepackage{mathpazo} 
\usepackage{appendix}
\usepackage{booktabs}

\usepackage[pdftex]{graphicx} 
 
\usepackage{geometry}
\usepackage{macros}
\usepackage[acronym,nonumberlist,nopostdot]{glossaries}
\glsdisablehyper
\newacronym{gp}{GP}{Gaussian processe}
\newacronym{dnn}{DNN}{Deep Neural Network}
\newacronym{dgp}{DGP}{Deep Gaussian Process}
\newacronym{qoi}{QoI}{quantity of interest}
\newacronym{mse}{MSE}{mean square error}
\newacronym{ard}{ARD}{Automatic Relevancy Determination}
\newacronym{rmse}{RMSE}{root mean square error}
\newacronym{svgp}{SVGP}{Sparse Variational Gaussian Process}
\newacronym{koh}{KOH}{Kennedy \& O'Hagan}
\newacronym{nargp}{NARGP}{Nonlinear Autoregressive Gaussian Process}
\usepackage[backend=bibtex,style=ieee,natbib=true]{biblatex}
\title{Multifidelity-Augmented Gaussian Process Inputs \\ for Surrogate Modeling from Scarce Data}

\usepackage{authblk}

\author[1]{Atticus Rex}
\author[1,2]{Elizabeth Qian}
\author[3]{David Peterson}

\affil[1]{\small School of Aerospace Engineering, Georgia Institute of Technology, Atlanta, GA}
\affil[2]{\small School of Computational Science \& Engineering, Georgia Institute of Technology, Atlanta, GA}
\affil[3]{\small Air Force Research Lab, Wright-Patterson AFB, Dayton, Ohio}

\begin{document} 
\date{}
\maketitle

\begin{abstract} 
\noindent Supervised machine learning describes the practice of fitting a parameterized model to labeled input-output data. Supervised machine learning methods have demonstrated promise in learning efficient surrogate models that can (partially) replace expensive high-fidelity models, making many-query analyses, such as optimization, uncertainty quantification, and inference, tractable. However, when training data must be obtained through the evaluation of an expensive model or experiment, the amount of training data that can be obtained is often limited, which can make learned surrogate models unreliable. In many engineering and scientific settings, cheaper low-fidelity models may be available, for example arising from simplified physics modeling or coarse grids. These models may be used to generate additional low-fidelity training data. The goal of multifidelity machine learning is to use both high- and low-fidelity training data to learn a surrogate model which is cheaper to evaluate than the high-fidelity model, but more accurate than any available low-fidelity model. This work proposes a new multifidelity training approach for Gaussian process regression which uses low-fidelity data to define additional features that augment the input space of the learned model. Similarly to cokriging estimators, the proposed approach conditions the high-fidelity surrogate model on the predictions of all available low-fidelity surrogate models, while benefiting from the computational efficiency of autoregressive estimators. Numerical experiments on several test problems demonstrate both increased predictive accuracy and reduced computational cost relative to the state of the art.
\end{abstract}

\section{Introduction} \label{introduction}

Many scientific and engineering decisions depend on repeated physical experiments or evaluations of a computational model. This is known as \emph{many-query} analysis, which arises in optimization \cite{koziel_multi-level_2013}, uncertainty quantification \cite{zhang_multi-fidelity_2024}, and inverse problems \cite{cui_datadriven_2015, qian_model_2022}. In many practical settings, evaluating such experiments or computational models is expensive, making many-query analysis intractable. Obtaining cheap, accurate \emph{surrogate models} is often required to enable many-query analysis in complex applications. This work focuses on learning surrogate models from labeled input-output data via \emph{supervised machine learning} \cite{wang_recent_2022, kennedy_bayesian_2001, sacks_design_1989}. We consider an expensive experiment or simulation to be a function which maps known system inputs to some unknown \gls{qoi}. This function is referred to as the \emph{high-fidelity} model and is considered to be the most accurate representation of the true system available. When high-fidelity data are plentiful, standard supervised learning algorithms can train reliable surrogate models \cite{wang_experimental_2022}. However when high-fidelity data are scarce, the trained surrogate model can suffer from overfitting, poor generalization, inability to capture complex nonlinear behavior, and high model uncertainty \cite{vapnik_overview_1999, hastie_elements_2009, rasmussen_gaussian_2008}. 

To combat this data scarcity problem, cheaper approximations of the high-fidelity model, known as \emph{low-fidelity} models, may be used to generate more training data. Low-fidelity models may arise from simplifying assumptions (e.g., switching from a 3-D spatial model to a 2-D or 1-D model, enforcing steady-state conditions, or linearizing about an equilibrium), or the use of lower resolution discretizations. The goal of \emph{multifidelity machine learning} is to combine high- and low-fidelity data to train a surrogate model with higher accuracy than any low-fidelity model, but lower cost than the high-fidelity model \cite{peherstorfer_survey_2018, qian_multifidelity_2025, fernandez-godino_review_2023, brevault_overview_2020, gorodetsky_mfnets_2020, howard_multifidelity_2023, howard_multifidelity_2024, lu_multifidelity_2022, heinlein_multifidelity_2024}. Beyond machine learning, multifidelity approaches have been developed for uncertainty quantification \cite{giles_multilevel_2015, gorodetsky_generalized_2020, peherstorfer_optimal_2016, gorodetsky_grouped_2024, schaden_multilevel_2020, qian_multifidelity_2018}, statistical inference \cite{popov_multifidelity_2021, peherstorfer_transport-based_2018, catanach_bayesian_2020, christen_markov_2005, cai_multi-fidelity_2022, warne_multifidelity_2022}, and optimization \cite{keil_relaxed_2024, qian_certified_2017, klein_multi-fidelity_2025, li_multi-fidelity_2024, song_general_2019, forrester_multi-fidelity_2007, zahr_multilevel_2017, amsallem_design_2015, yano_optimization-based_2012}. 

In this work, we select candidate surrogate models from Gaussian distributions of functions, known as \glspl{gp} \cite{rasmussen_gaussian_2008}. \gls{gp} \emph{regression} (also known as kriging) is a probabilistic supervised learning technique which emerged from geostatistics in the 20th century as a way to approximate nonlinear functions from data \cite{krige_statistical_1951, matheron_principles_1963, kennedy_predicting_2000, neal_priors_1996, hornik_multilayer_1989, rasmussen_gaussian_2008}. An important limitation of \gls{gp} regression is that with $N$ training data points, the cost to train an exact \gls{gp} regression model grows at a rate of $\setO(N^3)$, the result of solving a size-$N$ linear system \cite{rasmussen_gaussian_2008}. Another challenge in GP regression is choosing the hyperparameters that define the Gaussian distribution of functions from which we choose a surrogate model. This practice, known as hyperparameter optimization, is usually iterative, requiring an $\setO(N^3)$ linear solve at each iteration~\cite{rasmussen_evaluation_1997}. Recent work modifies the standard form of \glspl{gp} as defined in \cite{rasmussen_gaussian_2008} to estimate vector-valued functions \cite{myers_matrix_1982, goulard_linear_1992, alvarez_kernels_2012, bonilla_multi-task_2007}, handle high-dimensional inputs \cite{xu_standard_2025, zhou_kernel_2020, snelson_variable_2006, tripathy_gaussian_2016, wilson_deep_2015}, choose optimal hyperparameters \cite{rasmussen_evaluation_1997,berger_bayesian_1985, lalchand_sparse_2022, stein_interpolation_1999, snoek_practical_2012}, incorporate non-Gaussian priors \cite{paciorek_spatial_2006, kindap_non-gaussian_2022,damianou_deep_2013, rasmussen_evaluation_1997}, and handle large amounts of training data \cite{smola_sparse_2000, williams_using_2000, titsias_variational_2009, fine_efficient_2001, wilson_kernel_2015, rahimi_random_2007, dai_scalable_2014}. Additionally, we emphasize that GPs are fundamentally statistical prediction models which do not incorporate any underlying physics. However, physics-informed GPs which satisfy governing physical equations are considered in \cite{raissi_machine_2017, raissi_numerical_2018, pfortner_physics-informed_2024}.

GP-based multifidelity surrogate modeling approaches can typically be classified as either cokriging or autoregressive. Cokriging estimators take the form of multi-output \glspl{gp} \cite{kennedy_predicting_2000, gratiet_multi-fidelity_2013, myers_matrix_1982, goulard_linear_1992}. This approach fits a joint Gaussian distribution to all levels of fidelity simultaneously, but incurs large offline training and online prediction costs when the number of training data points is large (which is likely the case at lower fidelity levels). Additionally, cokriging methods require a unique kernel covariance function to describe each pair of fidelity levels, causing the hyperparameter count to grow at a rate of $\setO(K^2)$, where $K$ is the number of levels of fidelity. This often results in poor-conditioning of the kernel matrices in offline training and hyperparameter optimization \cite{alvarez_kernels_2012}. To reduce the number of kernel hyperparameters, many works employ linear models of coregionalization (LMCs) which use linear combinations of shared kernels to relate the levels of fidelity \cite{garland_multi-fidelity_2020, goulard_linear_1992}. While LMCs reduce computational complexity, they also limit the high-fidelity surrogate model form to linear combinations of the low-fidelity models. Encoding the inputs with \glspl{dnn} as described in \cite{raissi_deep_2016} allows for expressive kernels, but does not provide accurate uncertainty estimates in online prediction and still requires the inversion of large multi-output kernel matrices. 

As an alternative to cokriging, autoregressive estimators train separate fidelity-specific surrogate models and then combine them to form a multifidelity model. The method described by Kennedy \& O'Hagan in \cite{kennedy_predicting_2000} uses a linear mapping between fidelity levels and corrects for nonlinear discrepancies with a \gls{gp} at each level. Since its publication, many extensions to the Kennedy O'Hagan predictor have emerged, including formulations for autoregressive DNNs instead of \glspl{gp} \cite{tang_combined_2024} and the use of advanced experimental design methods \cite{zhou_generalized_2020, zhou_sequential_2017, yang_gaussian_2024}. The nonlinear autoregressive \gls{gp} (NARGP) approach described in \cite{perdikaris_nonlinear_2017} uses a general nonlinear function to map one level of fidelity to the next-highest level. This idea is built upon as a multifidelity deep Gaussian process (MF-DGP) in \cite{cutajar_deep_2019}.  The key innovation of the NARGP and MF-DGP methods is the nonlinear combination of system inputs with low-fidelity model evaluations in a single \gls{gp} kernel. These approaches empirically outperform Kennedy O'Hagan and cokriging-style estimators, especially in the presence of nonlinear relationships between levels of fidelity \cite{perdikaris_nonlinear_2017, cutajar_deep_2019}.

When many levels of fidelity are available, the autoregressive estimators invert smaller kernel matrices compared to cokriging estimators, making them computationally cheaper in offline training. However, most still require \glspl{gp} as low-fidelity surrogate models. At lower fidelity levels, there may be millions of available training data points because the underlying computer models are cheap to evaluate. This makes training \glspl{gp} in the multifidelity setting an especially expensive task. We emphasize that these existing methods are \emph{first-order} autoregressive strategies; each fidelity-specific surrogate model is conditioned only on the predictions of the next-lowest-fidelity surrogate model. Information from any lower-fidelity surrogate models must propagate through the next-lowest-fidelity surrogate model. In this work, we increase predictive accuracy by employing a \emph{higher-order} autoregressive strategy in which each surrogate model is directly conditioned on the predictions of \emph{all} available lower-fidelity surrogate models.

We propose a novel multifidelity \gls{gp}-based surrogate modeling strategy which trains one high-fidelity \gls{gp} surrogate model from system inputs augmented with features defined by the available low-fidelity data. These multifidelity-augmented features are formed recursively with the training of each low-fidelity surrogate model; for each fidelity-specific dataset, we train a surrogate model whose predictions become a new feature on which the next surrogate model is trained. In contrast to the NARGP and MF-DGP approaches~\cite{perdikaris_nonlinear_2017, cutajar_deep_2019}, where only the next-lowest model is integrated into each \gls{gp} kernel, we use predictions from \emph{all} available low-fidelity surrogate models. In the proposed method, we also relax the requirement that the low-fidelity surrogate models are \glspl{gp}, which substantially reduces offline training and online prediction cost. In this work, we consider \glspl{gp} as described in \cite{rasmussen_gaussian_2008}, but emphasize that the many modifications to this canonical formulation may be applied to the proposed method. Our contributions are: 
\begin{enumerate}
    \item We propose a novel GP-based multifidelity data-driven surrogate modeling technique which uses low-fidelity data to define features that augment the input space of the learned surrogate model. 
    \item We provide analysis of the method, characterizing its computational cost and guaranteeing the existence of a marginal likelihood at least as high as existing methods under mild assumptions.
    \item We demonstrate numerically that the method is both more accurate and cost efficient than existing single- and multifidelity machine learning methods on surrogate modeling problems drawn from chemical kinetics modeling and aerospace propulsion. 
\end{enumerate}

The remainder of this paper is organized as follows: \Cref{background} introduces the problem statement and relevant background; \Cref{new-method} presents the new method; \Cref{results} provides numerical experiments; \Cref{conclusion} concludes.

\section{Background} \label{background}
In this section we provide necessary background and context for the proposed method. \Cref{problem-statement} defines the multifidelity problem statement and mathematical notation, \Cref{gpr_summary} introduces single-fidelity Gaussian process regression (kriging), and \Cref{cokriging-estimators,autoregressive-estimators} discuss the existing cokriging and autoregressive approaches to multifidelity machine learning.

\subsection{Problem Statement and Notation} \label{problem-statement}
Let $\{y_{l}:\R^d \rightarrow \R \}_{{l}=1}^K$ denote a set of $K$ models of varying levels of fidelity, which map system inputs $\bx \in \R^d$ to a scalar quantity of interest.
At a given level ${l}$, we have a dataset $\setD_{l} = (\bX_{l}, \by_{l})$ where $\bX_{l} = \{\bx_1^{({l})}, \dots, \bx_{N_{l}}^{({l})} \}$, $\by_{l} = [y_{l}(\bx_1^{({l})}), \dots, y_{l}(\bx_{N_{l}}^{({l})})]^\top \in \R^{N_{l}}$, and $N_{l}$ is the number of training data points available. We make no noise-free assumptions about the output data nor any nesting assumptions about the input data. We emphasize that the only ordering assumption is that $y_1$ is highest-fidelity, but $y_2, \dots, y_K$ may be arbitrarily ordered with respect to fidelity. We consider the black-box setting in which we only have access to the data, not to the true models.  Our goal is to use $\setD_1, \dots, \setD_K$ to learn a predictive surrogate model $h_1:\R^d \rightarrow \R$ which accurately approximates the high-fidelity model $y_1$. 

\subsection{Single-Fidelity Gaussian Process Regression} \label{gpr_summary}

Let us consider the high-fidelity training outputs $\by_1$ as the true model $y_1$ evaluated on the inputs $\bX_1$ and corrupted by Gaussian white noise:
\begin{align*}
\by_1 \sim \setN \left( y_1(\bX_1),  \sigma_1^2 \bI \right),
\end{align*}
where $\sigma_1^2 \in \R_+$ is the variance of the noise. Our goal is to choose an $h_1$ to match the uncorrupted $y_1$ with minimal error. We now define a set of functions $\setH_\theta$ equipped with a Gaussian probability density function $p(h_1|\theta)$ where $\theta \in \Theta \subseteq \R^p$ is the set of hyperparameters. This density is the functional \emph{prior} distribution of the surrogate model $h_1$. Assuming $y_1 \in \setH_\theta$, Gaussian process regression seeks the \emph{posterior} distribution of $h_1$ given a noisy set of observations of the true model $y_1$. Both the set of hypothesis functions $\setH_\theta$ and the prior $p(h_1|\theta)$ are fully specified by a mean function, $\mu(\cdot; \theta):\R^d  \times \Theta \rightarrow \R$ and a symmetric positive definite kernel covariance function, $k(\cdot, \cdot; \theta):\R^d \times \R^d \times \Theta \rightarrow \R$. The mean function is typically chosen to capture general trends in the data (e.g., zero, constant, linear, quadratic, etc.). 

In this work, we will use Radial Basis Function (RBF) kernels, often called Gaussian or Squared Exponential kernels, a common choice for approximating smooth, nonlinear functions. Specifically, we will use an anisotropic RBF kernel known as the \gls{ard} kernel: 
\begin{align} \label{rbf-kernel}
    k(\bx, \bx'; b, \blambda) = b^2 \cdot \exp\left(-\frac{1}{2} \sum_{i=1}^d \frac{( \bx_i - \bx'_i)^2}{\lambda_i^2}  \right), 
\end{align} 
where $\blambda =\{\lambda_1, \dots, \lambda_d\}$ is a set of real scalars and $\bx_i$ denotes the $i$th entry of $\bx$. When the hyperparameters $\blambda$ are optimized, they automatically adjust to determine the relevance of each entry of $\bx$. If an entry $\bx_i$ provides useful information about $h_1(\bx)$, then $\lambda_i$ will be small. If $\bx_i$ provides no information about $h_1(\bx)$, then $\lambda_i \rightarrow \infty$. If an ARD kernel is used to define the functional prior $p(h_1|\theta)$, all functions contained in this prior are continuous and infinitely differentiable \cite{rasmussen_gaussian_2008}.

The hyperparameters $\theta$ are the union of the mean and kernel parameters. For brevity, we sometimes elect to drop the explicit dependence on $\theta$ in $k$ and $\mu$.  Let $\bK = \bK(\theta) = k (\bX_1, \bX_1; \theta) = k(\bX_1, \bX_1)$ denote the $N \times N$ kernel matrix whose entries are evaluations of the kernel function $k(\cdot, \cdot; \theta)$ at each pair of training inputs (e.g., $\bK_{ij} = k(\bx_i, \bx_j; \theta)$ for $\bx_i, \bx_j \in \bX_1$). Let $\bx'$ be an arbitrary test input at which we seek to approximate the true model $y_1(\bx')$. We assume $\by_1$ has been sampled from $\setH_\theta$ with additional Gaussian white noise of variance $\sigma_1^2$. This means $\by_1$ and $h_1(\bx')$ are sampled according to the following joint normal distribution:
\begin{align} \label{eqn:gpr_bivariate}
    \begin{bmatrix}
        \by_1 \\ h_1(\bx') 
    \end{bmatrix} \sim \mathcal{N} \left( \begin{bmatrix}
        \mu(\bX_1) \\ \mu(\bx') 
    \end{bmatrix}, \begin{bmatrix}
        k(\bX_1, \bX_1) + \sigma_1^2 \bI & k(\bX_1, \bx') \\ 
        k(\bx', \bX_1) & k(\bx', \bx') 
    \end{bmatrix} \right). 
\end{align}
If $h_1(\bx')$ varies according to \cref{eqn:gpr_bivariate}, the conditional distribution of $h_1(\bx')$ given a training dataset, $\setD_1 = (\bX_1, \by_1) $, is also Gaussian:   
\begin{subequations} \label{eqn:gpr-posterior}
    \begin{align}
        p(h_1(\bx')|\by_1) = \mathcal{N}&\left(\E{h_1(\bx') | \by_1},  \Var{h_1(\bx') | \by_1}\right), \\ 
        \E{h_1(\bx') | \by_1} &= k(\bx', \bX_1) \left( k(\bX_1, \bX_1) + \sigma_1^2 \bI \right)^{-1} \left( \by_1 - \mu(\bX_1) \right) + \mu(\bx'),  \label{analytical-mean} \\ 
        \Var{h_1(\bx') | \by_1} &= k(\bx', \bx') - k(\bx', \bX_1) \left( k(\bX_1, \bX_1) + \sigma_1^2 \bI \right)^{-1} k(\bX_1, \bx').  \label{analytical-variance}
    \end{align}
\end{subequations}
\Cref{eqn:gpr-posterior} describes the Bayesian posterior, $p(h_1|\by_1)$ of the surrogate model $h_1$. The \emph{marginal likelihood} (also called the model evidence) of the training outputs $\by_1$ given the hyperparameters of the model is 
\begin{subequations}
\begin{align*}
    p(\by_1|\theta) &= \int_{\setH_\theta} p(\by_1 | h_1, \theta) p(h_1| \theta) d h_1 \\ 
    &= \frac{1}{\sqrt{(2 \pi)^{N_1} |\bK(\theta) + \sigma_1^2 \bI|}} \exp \left( -\frac{1}{2} (\by_1 - \mu(\bX_1; \theta))^\top (\bK(\theta) + \sigma_1^2 \bI)^{-1} (\by_1 - \mu(\bX_1; \theta )) \right),
\end{align*}
\end{subequations}
where $|\cdot|$ denotes the matrix determinant. Optimal mean and kernel hyperparameters, as well as white noise variance, can be obtained by  maximizing the marginal likelihood (a practice known as Empirical Bayes or Type-II maximum likelihood estimation): 
\begin{subequations}\label{eqn:gpr-objective-func}
\begin{align} 
    \argmax{\theta, \sigma_1}\hspace{4pt} \log p(\by_1 | \theta, \sigma_1) &=  \argmin{\theta, \sigma_1} \hspace{6pt} \frac{1}{2} \left[ \Tilde{\by}_1(\theta) ^\top (\bK(\theta) + \sigma_1^2 \bI)^{-1} \Tilde{\by}_1(\theta) +\log|\bK(\theta) + \sigma_1^2 \bI| + N_1 \log (2 \pi) \right]\\ &= \argmin{\theta, \sigma_1} \hspace{6pt} \underbrace{\hspace{3pt} \Tilde{\by}_1(\theta) ^\top (\bK(\theta) + \sigma_1^2 \bI)^{-1} \Tilde{\by}_1(\theta)}_{\text{data fit incentive}} + \underbrace{\log|\bK(\theta) + \sigma_1^2 \bI|}_{\text{model complexity penalty}}, \label{eqn:gpr-minimization}
\end{align}
\end{subequations}
where $\Tilde{\by}_1(\theta) =\by_1 - \mu(\bX_1; \theta)$. \Cref{eqn:gpr-minimization} is comprised of a term which incentivizes how well $\mu(\bX_1; \theta)$ matches $\by_1$ on average and a term which penalizes the complexity of the functions contained within $\setH_\theta$. Optimizing \cref{eqn:gpr-objective-func} is usually solved through iterative constrained gradient-based optimization, and can be numerically challenging as a result of convergence to local extrema and poorly conditioned kernel matrices \cite{rasmussen_gaussian_2008}. 

Because \glspl{gp} are used throughout this paper as a solution to supervised machine learning problems, we will often use the following abbreviation: 
\begin{align*}
    \setG \setP (\bX, \by, \mu, k), 
\end{align*}
where $\bX$ is the set of training inputs, $\by$ is the vector of training outputs, $\mu$ is the mean function, and $k$ is the kernel covariance function. 

\subsection{Cokriging Estimators} \label{cokriging-estimators}

Cokriging is a multifidelity machine learning method which relates the levels of fidelity using a jointly distributed (multi-output) \gls{gp} \cite{myers_matrix_1982, xiao_extended_2018, goulard_linear_1992}:
\begin{align} \label{eqn:cokriging-prior}
    \begin{bmatrix}
        \by_1 \\ \vdots \\ \by_K
    \end{bmatrix} \sim \setN \left( \begin{bmatrix}
        \mu_1(\bX_1; \theta_{1}) \\ \vdots \\ \mu_K(\bX_K; \theta_K)
    \end{bmatrix}, \begin{bmatrix}
        k_{11}(\bX_1, \bX_1; \theta_{11}) + \sigma_1^2 \bI & \dots & k_{1K}(\bX_1, \bX_K; \theta_{1K})   \\ 
        \vdots & \ddots & \vdots  \\ 
        k_{K1}(\bX_K, \bX_1; \theta_{K1}) & \dots & k_{KK}(\bX_K, \bX_K; \theta_{KK}) + \sigma_K^2  \\ 
    \end{bmatrix}\right)
\end{align}
where $\sigma_{l}^2$ is the variance of Gaussian white noise at level ${l}$, $\mu_{l}$ is the mean function at level ${l}$, and $k_{ij}$ is the kernel function which relates level $i$ to level $j$. The function-space interpretation of cokriging not only provides us with the probabilistic uncertainty estimation of \glspl{gp}, but also a principled approach for hyperparameter optimization. If we denote $\boldsymbol{\theta}$ as the set of hyperparameters across all kernels and mean functions, and $\bs = \{\sigma_1^2, \dots, \sigma_K^2\}$ as the set of noise variances used in \cref{eqn:cokriging-prior}, we can maximize the the model evidence as we do in \cref{eqn:gpr-objective-func}: 
\begin{subequations}
    \begin{align*}
        \boldsymbol{\theta}^*, \bs^* &= \argmax{\boldsymbol{\theta}, \bs} \quad \log p(\by_1, \dots, \by_K | \boldsymbol{\theta}, \bs) \\ 
        &=  \argmin{\boldsymbol{\theta}, \bs} \quad \Tilde{\bY}^\top (\Tilde{\bK}(\btheta) + \Tilde{\bSigma})^{-1} \Tilde{\bY} + \log(|\Tilde{\bK}(\btheta) + \Tilde{\bSigma}|), 
    \end{align*}
\end{subequations}
where the following abbreviations have been made: 
\begin{align*}
    \Tilde{\bY} = \begin{bmatrix}
        \by_1 - \mu_1(\bX_1; \theta_1) \\ \vdots \\ \by_K - \mu_K(\bX_K; \theta_K)  
    \end{bmatrix}, &\quad \Tilde{\bK}(\btheta) = \begin{bmatrix}
        \bK_{11}(\theta_{11})  & \dots & \bK_{1K}(\theta_{1K})   \\ 
        \vdots & \ddots & \vdots  \\ 
        \bK_{K1}(\theta_{K1}) & \dots & \bK_{KK}(\theta_{KK})  \\ 
    \end{bmatrix}, \\ 
    \text{and} \quad \Tilde{\bSigma} &= \text{block-diag}(\sigma_1^2 \bI, \dots, \sigma_K^2 \bI ). 
\end{align*}
The predictive posterior distribution for high-fidelity model evaluations at an unseen test input $\bx'$ is Gaussian, defined by 
\begin{subequations}
\begin{align*}
    \E{h_1(\bx')|\by_{1:K}} &= \Tilde{\bK}' \left( \Tilde{\bK}(\btheta) + \Tilde{\bSigma} \right)^{-1} \Tilde{\bY} + \mu_1(\bx'; \theta_1) \quad \text{and} 
    \\ 
    \V{h_1(\bx')|\by_{1:K}} &= k_{11}(\bx', \bx') - \Tilde{\bK}' \left( \Tilde{\bK}(\btheta) + \Tilde{\bSigma} \right)^{-1} \Tilde{\bK}'^\top, 
\end{align*}
\end{subequations}
where $\Tilde{\bK}' = \begin{bmatrix}
    k_{11} (\bx', \bX_1) & \dots & k_{1K}(\bx', \bX_K)  
\end{bmatrix}$. The cokriging estimator is sensitive to the selection of the kernel hyperparameters $\btheta$ and prone to numerical instability from poorly conditioned kernel matrices. For this reason, LMCs are often used to express kernel $k_{ij}$ as a linear combination of shared kernels across all levels of fidelity: 
\begin{align*}
    k_{ij}(\bx, \bx'; \btheta) = \sum_{r=1}^R b_{ij}^{(r)} k(\bx, \bx'; \theta_r). 
\end{align*}
This kernel formulation prevents the number of kernels (and therefore the number of hyperparameters) from scaling quadratically with the number of levels of fidelity \cite{garland_multi-fidelity_2020, goulard_linear_1992}. This simplification limits the model form to linear mappings between levels of fidelity, thereby preventing nonlinear combinations of low-fidelity model evaluations from informing high-fidelity predictions. Physics-informed cokriging models have been considered for two-level cases in \cite{yang_physics-informed_2019, yang_when_2020}. 

\subsection{Autoregressive Estimators} \label{autoregressive-estimators}

Autoregressive estimators approximate the levels of fidelity sequentially, starting at level $K$ (lowest-fidelity), and apply some mapping from level ${l}+1$ to level ${l}$ until the highest fidelity is reached. In contrast to cokriging models, each individual surrogate model $h_{l}:\R^d \rightarrow \R$ is trained separately instead of jointly. An example of such an approach is that proposed by \gls{koh} in \cite{kennedy_predicting_2000}. The \gls{koh} method adapts surrogate model $h_{{l}+1}$ to match the next-highest level $h_{{l}}$: 
\begin{align} \label{eqn:koh_model}
    h_{{l}}(\bx) = {\rho_l} \cdot h_{{l}+1}  (\bx) + {\delta}_{l} (\bx), 
\end{align}
\noindent where $\rho_l \in \R$ is a constant scaling $h_{{l}+1}$ to $h_{{l}}$. The function $\delta:\R^d \rightarrow \R$ accounts for nonlinear variation not captured by $\rho_l$. In \cite{kennedy_predicting_2000}, $\delta(\bx)$ is a \gls{gp}: 
\begin{align*}
    \delta_{l} = \mathcal{GP}\left(  \bX_{l}, \left[ \by_{{l}}-\rho_l \cdot y_{{l}+1}(\bX_{l}) \right], \mu_l = 0, k_{{l}} \right), 
\end{align*}
\noindent where $\mu_{l}$ and $k_{{l}}$ are the level-specific mean and kernel functions. The parameters $\rho_l$ and $\theta$ at each level are chosen via maximum marginal likelihood estimation similarly to that described in \cref{eqn:gpr-objective-func}. Sometimes, authors elect to parameterize $\rho_l$ as a function $\rho_l(\bx)$ \cite{gratiet_multi-fidelity_2013}. Generally, no constraints are placed on the coefficients $\rho_1, \dots, \rho_{K-1}$, however, when many levels of fidelity are present, keeping $\rho_l \approx 1$ prevents significant exponential growth or decay of posterior means and variances \cite{gratiet_multi-fidelity_2013}. The estimator described in \cref{eqn:koh_model} can also be formulated as a multi-output GP where $\delta_{l}(\bx) \sim \setN(0, k_{{l}})$ and $h_{{l}+1} \sim \setN(0, k_{{l}+1})$:
\begin{align} \label{eqn:koh-cokriging}
    \begin{bmatrix}
        h_{{l}+1}(\bx) \\ h_{{l}}(\bx) 
    \end{bmatrix} \sim \setN \left( \begin{bmatrix}
        0 \\ 0  
    \end{bmatrix}, \begin{bmatrix}
        k_{{l}+1}(\bx, \bx') & \rho_l \cdot k_{{l}+1}(\bx, \bx') \\ \rho_l \cdot k_{{l}+1}(\bx, \bx') & \rho_l^2 \cdot k_{{l}+1}(\bx, \bx') + k_{{l}}(\bx, \bx') 
    \end{bmatrix} \right).
\end{align}
This approach is also known as ``recursive cokriging'' \cite{gratiet_recursive_2014}. 

In contrast, the \gls{nargp} approach proposed in \cite{perdikaris_nonlinear_2017} combines the $\rho_l \cdot h_{{l}+1}(\bx)$ term into a single nonlinear transformation of $\bx$ and $h_{{l}+1}(\bx)$:
\begin{align*}
    h_{{l}} = \setG \setP \left( \begin{bmatrix}
        \bX_{{l}} & h_{{l}+1}(\bX_{{l}})
    \end{bmatrix}, \by_{{l}}, \mu=0, k_{l} \right).
\end{align*}
The consolidation of $\rho_l$, $h_{{l}+1}(\bx)$ and $\delta(\bx)$ results in the following multifidelity kernel: 
\begin{align*}
k_{l}(\bx, \bx') = k(\bx, \bx'; \theta_{p}) \cdot k(h_{{l}+1}(\bx), h_{{l}+1}(\bx'); \theta_{h}) + k(\bx, \bx'; \theta_{\delta}).
\end{align*}
Each of these kernels is parameterized by unique hyperparameters $\theta_{ p}$, $\theta_{h}$, and $\theta_{\delta}$ at each level of fidelity. Because the kernel takes evaluations of $h_{l+1}$ as inputs, these inputs are uncertain. However, under the assumption that the levels of fidelity are noiseless and the training data are nested such that $\bX_{l} \subset \bX_{{l}+1}$ (see \cite{perdikaris_nonlinear_2017, gratiet_recursive_2014}), the estimator can be trained at each fidelity level as a single-fidelity GP by maximum marginal likelihood estimation. Some work has considered autoregressive GP priors which incorporate underlying physics \cite{spitieris_bayesian_2023, yang_physics-informed_2019, yang_when_2020}. 

\section{Multifidelity-Augmented Gaussian Process Inputs} \label{new-method}

This section presents the proposed method. \Cref{method-overview} provides a step-by-step description of the method and justifies its theoretical utility, and \cref{analysis} analyzes the computational cost. 

\subsection{Method Overview} \label{method-overview}

The proposed method trains $K$ surrogate models to emulate each level of fidelity, starting from level $K$ (lowest-fidelity) and ending at level 1 (highest-fidelity). At each subsequent level ${l}$, the outputs from trained surrogate models $h_{{l}+1}$ through $h_K$ are used as features to provide additional information about the true model $y_{l}$. For pseudocode of the proposed method, refer to \Cref{algo:offline-train} for offline training and \Cref{algo:online-predict} for online prediction. Except for the highest level of fidelity, these surrogate models may be any regression method that provides a point estimate for the unknown function (e.g., linear regression, Deep Neural Networks, K-Nearest Neighbors, Random Forests, etc.). 

\begin{algorithm}[htb!] 
\caption{Offline Training}
\label{algo:offline-train}
\begin{algorithmic}[1] 
\State \textbf{Inputs:} multifidelity training data, $\{ \setD_{l} = (\bX_{l}, \by_{l})\}_{{l} = 1}^{K}$, kernel function $k$, mean function $\mu$
\State \textbf{Outputs:} trained surrogate models $h_1$ through $h_K$ for each level of fidelity \vspace{12pt}

\State Initialize the high-fidelity feature matrix $\bPhi_1 = \bX_1$ where $\bX_1 \in \R^{N_{l} \times d}$. 
\vspace{12pt}

\For{${l} = K, \dots, 2$ }
    \State Initialize $\bPhi_{l} = \bX_{l}$ 
    \For{$j = K, \dots, {l}+1$}
        \State Update feature matrix with predictions using model $j$: $\bPhi_{l} = [\bPhi_{l} \hspace{12pt} h_j(\bPhi_{l})]$
    \EndFor 
    \State Train regression model $h_{l}$ on training data $(\bPhi_{l}, \by_{l}) $
    \State Update high-fidelity features using model $l$: $\bPhi_1 = [\bPhi_1 \hspace{12pt} h_{l}(\bPhi_1)] $
\EndFor \vspace{12pt}

\State Train a GP regression model $h_1$ on $\bPhi_1$ and $\by_1$ with kernel $k$ and mean $\mu$ 
\vspace{12pt}

\State \textbf{return}  trained surrogate models $h_1, \dots, h_K$
\end{algorithmic}
\end{algorithm}

\begin{algorithm}[htb!] 
\caption{Online Prediction} 
\label{algo:online-predict}
\begin{algorithmic}[1] 
\State \textbf{Inputs:} set of testing inputs $\bX'$, trained regression models $h_1,\dots, h_K$.  
\State \textbf{Outputs:} posterior mean and variance of high-fidelity model predictions at $\bX'$ \vspace{12pt}

\State Initialize the testing features $\bPhi' = \bX'$
\vspace{12pt}

\For{${l} = K, \dots, 2$ }
    \State Update features : $\bPhi' = [\bPhi' \hspace{12pt} h_{l}(\bPhi')] $
\EndFor \vspace{12pt}
\State \textbf{return} posterior $\bmu, \bSigma$ from trained high-fidelity \gls{gp} model $h_1(\bPhi')$ (see \cref{eq:magpi-posterior}). 
\end{algorithmic}
\end{algorithm}

\newpage 

Let $\bPhi_K \in \R^{N_{K} \times d}$ be a data matrix such that each row is the transpose of a training input in $\bX_K$. To begin the offline training process, regression model $h_K:\R^d \rightarrow \R$ is trained on dataset $\setD_K = (\bPhi_K, \by_K)$.  Next, regression model $h_{K-1}:\R^{d+1} \rightarrow \R$ is trained using the feature matrix $\bPhi_{K-1}\in \R^{N_{K-1} \times (d + 1)}$ and outputs $\by_{K-1} \in \R^{N_{K-1}}$. The feature matrix $\bPhi_{K-1} $ is formed by horizontally concatenating inputs $\bX_{K-1}$ and evaluations of the trained model $h_K$ on $\bX_{K-1}$: 
\begin{align*}
    \bPhi_{K-1} = \begin{bmatrix}
        \phi_{K-1}(\bx_1)^\top \\ \vdots \\ \phi_{K-1}(\bx_{N_{K-1}})^\top 
    \end{bmatrix} = \begin{bmatrix}
        \bX_{K-1} & h_K(\bX_{K-1}) 
    \end{bmatrix} = \begin{bmatrix}
        \bx_1^\top & h_K(\bx_1) \\ 
        \vdots & \vdots \\ 
        \bx_{N_{K-1}}^\top & h_K(\bx_{N_{K-1}}) 
    \end{bmatrix}. 
\end{align*}
Each subsequent regression model $h_{l}$ is trained on a unique feature matrix $\bPhi_{l} \in \R^{N_{l} \times (d + K - {l})}$ and outputs $\by_{l} \in \R^{N_{l}}$. The rows of $\bPhi_{l}$ are evaluations of the features $\phi_{l}:\R^d \rightarrow R^{d+K-{l}}$ on each input in $\bX_{l}$. The features $\phi_{l}$ are evaluated recursively for any input $\bx$: 
\begin{align} \label{eq:recursive-features}
    \phi_{l}(\bx) = \begin{bmatrix}
        \phi_{{l}+1}(\bx) \\ h_{{l}+1}(\bx) 
    \end{bmatrix}, \quad \text{for} \quad{l} = 1, \dots, K-1 \quad \text{and} \quad \phi_K(\bx) = \bx \quad \text{(base case)}. 
\end{align}

After surrogate models $h_2, \dots, h_K$ are trained, the high-fidelity training features $\bPhi_1 \in \R^{N_1 \times (d+K-1)}$ combine the inputs $\bX_1$ and evaluations of \emph{all} trained low-fidelity surrogate models $h_2, \dots, h_K$ on each input $\bx_i \in \bX_1$: 
\begin{align*}
    \bPhi_1 = \begin{bmatrix} \phi_1(\bx_1)^\top \\ \vdots \\ \phi_1(\bx_{N_1})^\top \end{bmatrix} = \begin{bmatrix}
        \bx_1^\top & h_K(\bx_1) & \dots & h_2(\bx_1) \\ 
        \vdots & \vdots & \ddots & \vdots \\ 
        \bx_{N_1}^\top & h_K(\bx_{N_1}) & \dots & h_2(\bx_{N_1})
    \end{bmatrix}. 
\end{align*}
Once $\bPhi_1$ is computed, the high-fidelity \gls{gp} has its kernel hyperparameters optimized by maximizing the marginal likelihood: 
\begin{align*}
    \theta^*, \sigma_1^* = \argmin{\theta, \sigma} \hspace{6pt} \Tilde{\by}_1(\theta)^\top \left( k(\bPhi_1, \bPhi_1; \theta) + \sigma_1^2 \right)^{-1} \Tilde{\by}_1(\theta) + \log \left\vert k(\bPhi_1, \bPhi_1; \theta) + \sigma_1^2 \right\vert, 
\end{align*}
where $\Tilde{\by}_1(\theta) = \by_1 - \mu(\bX_1; \theta)$. Lastly, online predictions at unseen test inputs $\bX'$ are made by forming $\bPhi' = \phi_1(\bX')$ recursively using the low-fidelity model predictions, as in \cref{eq:recursive-features}. Once $\bPhi'$ is obtained, online predictions are made using the standard \gls{gp} posterior defined by  
\begin{subequations}
\begin{align} \label{eq:magpi-posterior}
    \E{h_1(\bX')|\by_1} &= k(\bPhi', \bPhi_1) \left( k(\bPhi_1, \bPhi_1) + \sigma^2 \right)^{-1} \left( \by_1 - \mu(\bPhi_1) \right) + \mu(\bPhi'), \\ 
    \V{h_1(\bX')|\by_1} &= k(\bPhi', \bPhi') - k(\bPhi', \bPhi_1) \left( k(\bPhi_1, \bPhi_1) + \sigma^2 \right)^{-1} k(\bPhi_1, \bPhi'). 
\end{align}
\end{subequations}

Because of the recursive structure of the estimator, the coefficients of each surrogate model are successively multiplied across the levels of fidelity, which may cause instability in high-fidelity predictions. However, unstable model parameters tend to produce highly inaccurate surrogate models; sequentially optimizing the model parameters at each level of fidelity mitigates the risk of such recursive instabilities and we do not observe instabilities in our numerical experiments in \Cref{results}. Additionally, because each set of features $\phi_l(\bx)$ uses a concatenation of the system inputs $\bx$ and model outputs $h_{l-1}, \dots, h_K$, the entries of $\phi_l(\bx)$ may span multiple orders of magnitude. This often results in numerically ill-conditioned parameter optimization problems. A standard technique to accelerate convergence is transforming each entry of $\phi_{l}(\bx)$ to have zero mean and unit variance over the training inputs $\bX_l$.

The proposed approach has several advantages over the existing methods discussed in \Cref{introduction}. In contrast to cokriging estimators, which solve large linear systems comprised of all high- and low-fidelity training data, we reduce cost by training surrogate models $h_K$ through $h_1$ sequentially, as the autoregressive estimators are trained. Switching to alternative low-fidelity surrogate models can further reduce cost in offline training and online prediction. Since the matrix $\bPhi_l$ contains all lower-fidelity surrogate model evaluations, the proposed method is not a first-order autoregressive strategy; rather, the order at a given level is exactly the number of available lower-fidelity surrogate models. The proposed method's integration of models $h_2$ through $h_K$ into a single GP kernel enables general nonlinear combinations of \emph{all} low-fidelity surrogate models to influence the predictions of the high-fidelity surrogate model $h_1$. Lastly, the mean function plays a subtle but crucial role in the estimator. Extrapolating outside of the training data is usually challenging for single-fidelity GP regression models; when an unseen test input is significantly far from the training inputs, the surrogate model simply outputs its functional prior \cite{rasmussen_gaussian_2008}. However, suppose $\mu:\R^{d+K-1} \times \Theta \rightarrow \R$ is a linear mean function of the form: 
\begin{align}
    \mu(\phi_1(\bx); \theta) = \sum_{i=1}^d \alpha_i \cdot x_i + \sum_{\ell=2}^K \beta_\ell \cdot h_\ell(\phi_\ell(\bx)) + \gamma, 
\end{align}
where $x_i$ is the $i$th entry of $\bx$, and the hyperparameters $\{ \alpha_i \in \R \}_{i=1}^d, \{\beta_\ell \}_{\ell=2}^K$, as well as $\gamma$ are calibrated via maximizing the log marginal likelihood. We emphasize that the mean function is itself a simple multifidelity predictive model which centers the GP's prior at a linear combination of the inputs and surrogate model evaluations. This informative prior allows the learned model to extrapolate beyond scarce high-fidelity training data into areas where low-fidelity training data is still plentiful. 

We now briefly justify the theoretical utility of adding low-fidelity features to the high-fidelity \glspl{gp}. We show that with specific classes of kernel and mean functions, adding more features as inputs guarantees the existence of hyperparameters which achieve at least as high a marginal likelihood as any trained \gls{gp} without the additional features. Let $\setK_1 = \{k_1:\R^d \times \R^d \rightarrow \R \}$ denote the set of kernel covariance functions and $\setM_1=\{\mu_1:\R^d \rightarrow \R \}$ denote the set of mean functions for $d$-dimensional inputs. Additionally, let $\setK_2=\{k_2:\R^{d+q} \times \R^{d+q} \rightarrow \R \}$ be the set of kernel covariance functions and $\setM_2 = \{\mu_2: \R^{d+q} \rightarrow \R \}$ be the set of mean functions for $(d+q)$-dimensional inputs formed by concatenating the original inputs with extra features, as in the proposed method. 

\begin{prp}
     Let $\setG_1 = \{(\mu_1, k_1) \hspace{6pt} \forall \mu_1 \in \setM_1, k_1 \in \setK_1\}$ and $\setG_2 = \{(\mu_2, k_2) \hspace{6pt} \forall \mu_2 \in \setM_2, k_2 \in \setK_2\}$. Let $p(\by|(\mu, k))$ be the marginal likelihood as defined in \cref{eqn:gpr-objective-func} for a given $\mu$ and $k$. For any set of training data: 
    \begin{align*}
        \setG_1 \subseteq \setG_2 \Rightarrow \sup_{(\mu_1, k_1) \in \setG_1} p(\by|(\mu_1, k_1))\leq \sup_{(\mu_2, k_2) \in \setG_2} p(\by|(\mu_2, k_2)). 
    \end{align*}
    \label{prp:marginal-likelihood-guarantee}
\end{prp}

\begin{proof}
    If $\setM_1 \subseteq \setM_2$ and $\setK_1 \subseteq \setK_2$, the additional features produce a larger set of mean and kernel functions while still containing $\setM_1$ and $\setK_1$. In the worst case, no mean or kernel function in $\setM_2$ and $\setK_2$ is found that achieves a higher marginal likelihood than the mean functions and kernels in $\setM_1$ and $\setK_1$. In the best case, a mean function $(\mu_2 \in \setM_2 \setminus \setM_1)$ and/or a kernel function $ (k_2 \in \setK_2 \setminus  \setK_1)$ is found which achieves a higher marginal likelihood. 
\end{proof}

\begin{rmk}
Combinations of standard kernels (e.g., RBF, ARD, polynomial, Laplace, Matern, spectral mixture) and mean functions (e.g., zero, constant, linear, quadratic) satisfy the property $\setG_1 \subseteq \setG_2$ and therefore satisfy \Cref{prp:marginal-likelihood-guarantee}. 
\end{rmk}

\subsection{Complexity Analysis} \label{analysis}
The time-complexity for offline training of the proposed method is $\setO(N_1^3 + \sum_{{l}=2}^K \tau_{\text{train}}^{({l})})$, where $\tau_{\text{train}}^{({l})}$ is the time complexity to train the surrogate model for fidelity ${l}$. The offline training space complexity is $\setO(N_1^2 + \sum_{{l}=2}^K s_{\text{train}}^{({l})})$, where $s_{\text{train}}^{({l})}$ is the space complexity to train the surrogate model for fidelity ${l}$. The online prediction time complexity is $\setO(N_1 + \sum_{{l}=2}^K \tau_{\text{predict}}^{({l})})$ where $\tau_{\text{predict}}^{({l})}$ is the time complexity to make online predictions with the surrogate model for fidelity ${l}$. The online prediction space complexity is $\setO(N_1 + \sum_{{l}=2}^K s_{\text{predict}}^{({l})})$ where $s_{\text{predict}}^{({l})}$ is the space complexity to make online predictions with the surrogate model for fidelity ${l}$. Because high-fidelity data are scarce (perhaps fewer than ten data points), the $\setO(N_1^3)$ training cost is not prohibitively expensive. Exact time and space complexities of the proposed method depend on the model forms of the low-fidelity surrogate models. Possible low-fidelity surrogate models include sparse variational GPs, Linear Regression, Deep Neural Networks, K-Nearest Neighbors, Random Forests, and Boosting models, none of which incur the cubic scaling with training data. A comparison of algorithmic time and space complexities for various standard regression models can be found in \Cref{alternative-algo-costs}. For low-dimensional inputs, simple regression algorithms like K-Nearest Neighbors incur orders of magnitude lower training costs than \glspl{gp} on large datasets.

In contrast, using \glspl{gp} as low-fidelity surrogate models is usually a computationally infeasible task since low-fidelity data are cheaper to obtain and therefore are usually available in much larger quantities, sometimes in the millions of training examples. The algorithmic cost for the autoregressive KOH and NARGP methods if full-rank \glspl{gp} are used is $\setO(\sum_{{l}=1}^{K} N_{l}^3)$. While some existing methods use approximate \glspl{gp} to remove the cubic scaling with training data, they can still be costly to implement and train. In the MF-DGP approach, the algorithmic cost is $\setO(k N Q^2)$ where $k$ is the number of iteration steps, $N$ is the total number of training data points, and $Q$ is the number of inducing points (this time-complexity excludes the number of Monte-Carlo samples used to approximate the likelihood). For cokriging models, the offline training time-complexity is $\setO\left(\left[\sum_{{l}=1}^K N_{l} \right]^3\right)$ and offline training space complexity is $\setO\left(\left[\sum_{{l}=1}^K N_{l} \right]^2 \right)$ . 

\section{Results} \label{results}

In this section, we present numerical results demonstrating the efficacy of the proposed method. \Cref{performance-metrics} defines the performance metrics, \Cref{synthetic-test-problem} compares the proposed method compared with Kennedy O'Hagan, NARGP, and single-fidelity kriging on a synthetic test problem, \Cref{laminar-flame-speed} compares how these methods extrapolate outside of a constrained high-fidelity design space in a chemical kinetics problem, and \Cref{velocimetry} demonstrates how these methods interpolate between scarce, nonlinear high-fidelity data in an aerospace propulsion problem. 

\subsection{Performance Metrics} 
\label{performance-metrics}

For each numerical problem in \Cref{results}, we consider three performance indices. First, the root means squared error (RMSE). For point-estimate predictions, the RMSE is calculated with
\begin{align*}
    \text{RMSE} = \sqrt{\frac{1}{M}\sum_{i=1}^M \left(h_1(\bx_i) - y_1(\bx_i) \right)^2}, 
\end{align*}
where there are $M$ unseen test inputs and $h_1(\bx_i)$ is the high-fidelity surrogate model evaluated at a single test input $\bx_i$. \glspl{gp}, however, do not output point estimate predictions; they output a \emph{predictive posterior} in the form of a Gaussian distribution. The RMSE across the entire predictive posterior is denoted 
\begin{align*}
    \text{RMSE for GPs} = \sqrt{\frac{1}{M}\sum_{i=1}^M \left( \V{h_1(\bx_i)|\setD} + \left[ \E{h_1(\bx_i)|\setD} - y_1(\bx_i) \right]^2 \right)}, 
\end{align*}
where $\V{h_1(\bx_i)|\setD}$ is the predictive posterior variance and $\E{h_1(\bx_i)|\setD}$ is the predictive posterior mean at test input $\bx_i$. We note that this formulation penalizes both the model uncertainty (variance) and the accuracy of the mean function. 

The second performance metric is $R^2$, the squared Pearson correlation coefficient between the model's predictions and true function. The Pearson correlation coefficient is computed with 
\begin{align} \label{eq:pearson}
R = \frac{\sum_{i=1}^{M} (h_1(\bx_i)- \bar{h}_1) (y_1(\bx_i) - \bar{y}_1)}{\sqrt{\sum_{i=1}^M (h_1(\bx_i) - \bar{h}_1)^2} \sqrt{\sum_{i=1}^M (y_1(\bx_i) - \bar{y}_1)^2}} \approx \frac{\Cov{h_1}{y_1}}{\sqrt{\V{h_1} \V{y_1}}}, 
\end{align}
where $\bar{h}_1=M^{-1} \sum_{i=1}^K h_1(\bx_i)$ and $\bar{y}_1=M^{-1} \sum_{i=1}^K y_1(\bx_i)$. We will use the GP posterior mean to compute $R^2$. The values of $R^2$ lie between zero and one, where $R^2=1$ indicates that the learned model matches the true model perfectly. 

The last performance metric for GP-based methods is the log-marginal likelihood (log ML). This is defined in \cref{eqn:gpr-objective-func} and is interpreted as the probability of sampling the training data given the GP prior induced by a given kernel and mean function. We can compare two GP regression models by computing the ratio of their marginal likelihoods: 
\begin{align*}
    \frac{p(\setD|\text{model 1})}{p(\setD|\text{model 2})} = \exp \left[ \log p(\setD|\text{model 1}) - \log p(\setD|\text{model 2}) \right]. 
\end{align*}
This ratio indicates how much \emph{more likely} training data are to be sampled from one GP prior compared to another. 

\subsection{Analytical Test Problem} \label{synthetic-test-problem}

To clearly illustrate the utility of the proposed method compared to existing multifidelity machine learning methods, we will first use a simple 1-D problem. The functions used, number of training data points at each fidelity level, and Pearson correlation (see \cref{eq:pearson}) between high- and low-fidelity models are outlined in \Cref{tab:synthetic-experiment}. Many existing methods assume high correlation between high- and low-fidelity models (often $>95\%$). We selected these simple functions such that the high-fidelity has only a 63.8\% and 41.8\% correlation with the medium- and low-fidelity models, respectively. Additionally, the high-fidelity model is a nonlinear combination of the medium- and low-fidelity models, which prevents the autoregressive estimators from linearly propagating necessary low-fidelity information through to the high-fidelity surrogate model. In this example, all surrogate models are \glspl{gp} and each \gls{gp} is parameterized using ARD kernels (see \cref{rbf-kernel}). In all examples, the proposed method is given a linear mean function while Kennedy O'Hagan, NARGP, and Kriging models are given a constant mean function to correct for vertical bias. We emphasize that these are canonical formulations; a full comparison of all the proposed variations (e.g., more complex mean functions \cite{stein_universal_1991}, active learning schemes \cite{he_active_2024}, fully-Bayesian treatments \cite{neal_bayesian_1996}, etc.) to these estimators is outside the scope of this paper. For maximum fairness across methods, each \gls{gp} had its hyperparameters (mean/kernel parameters and white noise variance) optimized using gradient-descent with the standard ADAM algorithm, and each was allowed to iterate until convergence (over 1,000 iterations without an improvement in marginal likelihood). The linear mean function parameters were initialized by matching training inputs to outputs via ordinary least-squares regression. The kernel parameters were initialized to generic positive values derived from the variance of the output data. The noise variances were initialized to $10^{-6}$ and increased such that the condition number of the kernel matrix did not exceed $10^5$. We emphasize that trained GP models can be sensitive to hyperparameter selection and the marginal likelihood function can be highly non-convex; hence, in the following results, we cannot guarantee globally optimal hyperparameters.

\begin{table}[htb!]
    \centering
    \begin{tabular}{ccccc}
    \toprule 
         \textbf{Fidelity} & \textbf{Function} & \textbf{\# of Data Points} & $\mathbf{R}$ \\ \midrule 
         High-Fidelity & $\sin(2 \pi \bx) \exp(-\bx) $ & 10 & 1.000 \\ 
         Medium-Fidelity & $\sin(2 \pi \bx)$ & 100 & 0.638 \\ 
         Low-Fidelity & $\exp(-\bx) $ & 250 & 0.417 \\ \bottomrule 
    \end{tabular}
    \caption{Experimental details for the analytical test problem. We indicate the fidelity level, function expression, number of training data points, and Pearson correlation coefficient with high-fidelity for each fidelity level listed \cref{eq:pearson} with the high-fidelity model.}
    \label{tab:synthetic-experiment}
\end{table}

\definecolor{myorange}{HTML}{bf7a02} 
\definecolor{mypurple}{HTML}{800080}
\begin{figure}[htb!]
    \centering
    \includegraphics[width=1\linewidth]{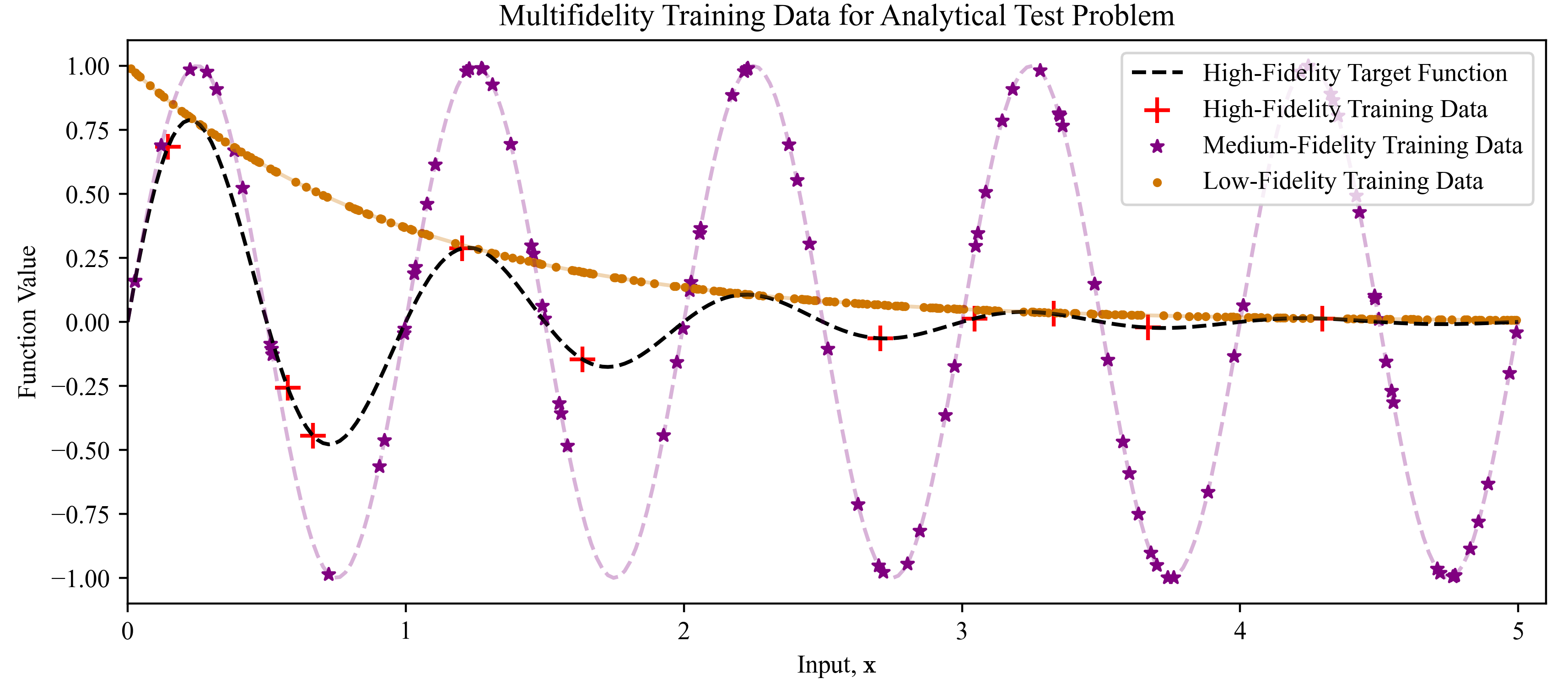}
    \caption{Training data collected by evaluating the high- (\textcolor{red}{\textbf{+}}), medium- (\textcolor{mypurple}{$\boldsymbol{\star}$}), and low-fidelity (\textcolor{myorange}{$\bullet$}) functions on uniform samples of $\bx$.}
    \label{fig:synthetic-data}
\end{figure}

\begin{figure}[htb!]
    \centering
    \includegraphics[width=1\linewidth]{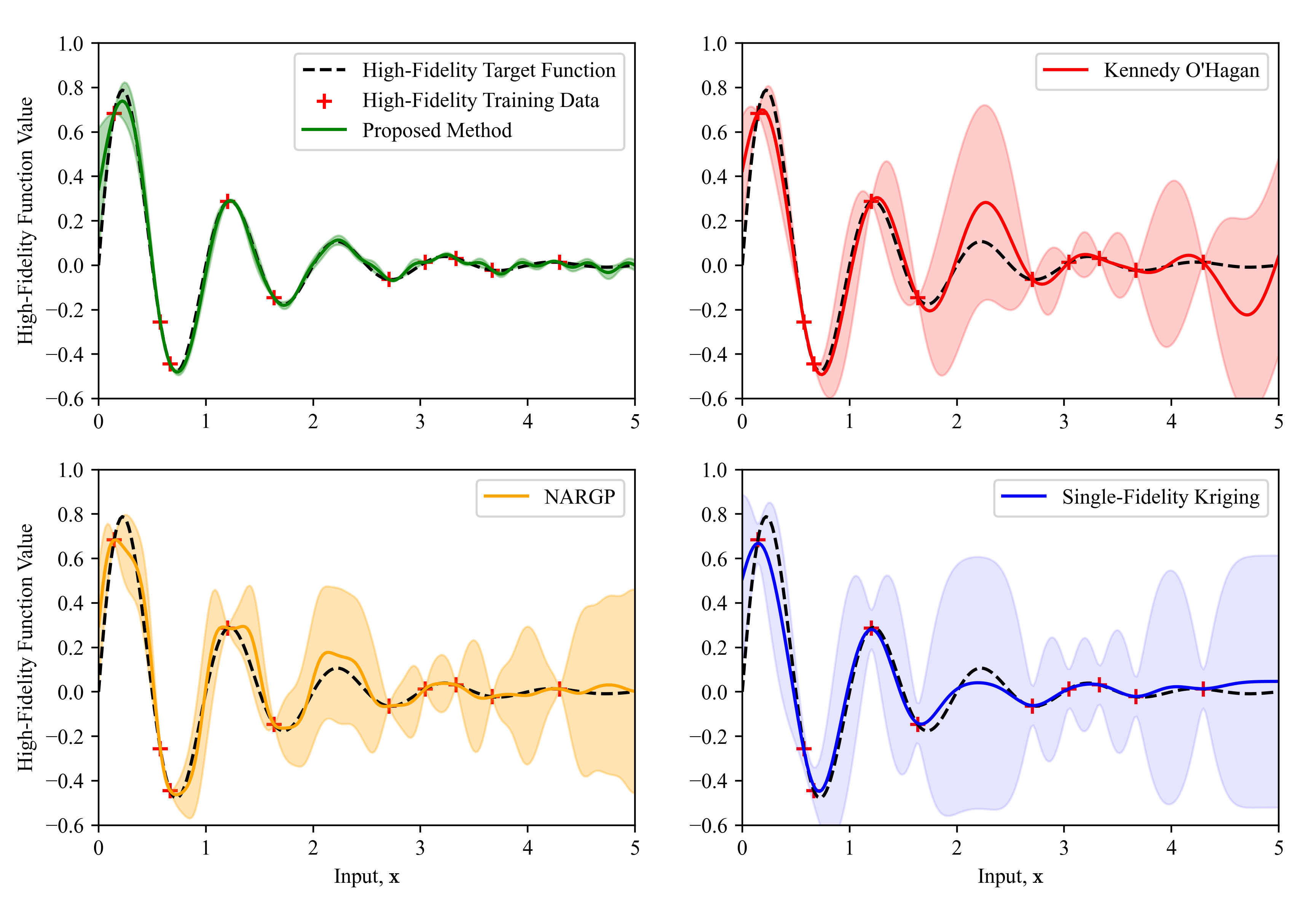}
    \caption{Results from trained models on the analytical test problem. The target functions are plotted with a black dashed line. (top left) The proposed model predictions, (top right) KOH model predictions, (bottom left) NARGP model predictions, (bottom right) single-fidelity kriging model predictions. The shaded regions represent $\pm 2 \sigma$ confidence intervals derived from the Gaussian posteriors of each estimator.}
    \label{fig:synthetic-results}
\end{figure}
\Cref{tab:synthetic-results} presents the performance metrics for the four methods tested in this analytical example. \Cref{fig:synthetic-data} plots the training data collected for each level of fidelity and \Cref{fig:synthetic-results} plots the test predictions from each model, together with the target function and training data. The shaded regions represent two standard deviations above and below ($\pm 2 \sigma$) the predictive mean. The proposed method produced an RMSE value over 50\% lower than the next lowest value (NARGP), an $R^2$ value approximately 10\%  higher than the next-highest value (NARGP), and a log ML roughly an order of magnitude higher than the next-highest value (NARGP). In terms of RMSE, both NARGP and Kennedy O'Hagan perform comparably to single-fidelity kriging. This is due to the limiting Markovian property present in the KOH and NARGP estimators; statistical information from the lowest-fidelity function is unable to propagate up the levels of fidelity and combine nonlinearly with the medium-fidelity function to accurately emulate the true high-fidelity function. 
\begin{table}[htb!]
    \centering
    \begin{tabular}{cccc}
        \toprule 
         \textbf{Approach} & \textbf{RMSE} & $\mathbf{R^2}$ & \textbf{log ML}\\ \midrule 
Proposed Method       &    \textbf{3.681e-02}   & \textbf{0.9731}  &  \textbf{10.6686} \\
Kennedy O'Hagan       &    8.951e-02   & 0.8560  &  2.1112 \\
NARGP                 &    5.765e-02   & 0.9415  &  1.2904 \\
Kriging               &    7.357e-02   & 0.8874  &  -0.8865 \\ \bottomrule 
    \end{tabular}
    \caption{Performance metrics on the analytical test example for each surrogate model evaluated at 250 linearly spaced inputs across the input space [0,5].}
    \label{tab:synthetic-results}
\end{table}

\subsection{Extrapolation of Constrained Laminar Flame Speed Data} \label{laminar-flame-speed}
In this example, we seek to predict laminar flame speed from temperature and equivalence ratio (see \cite{adusumilli_laminar_2021} for details). We will confine our high-fidelity data to a \emph{design space} such that the trained high-fidelity surrogate model must \emph{extrapolate} outside of this constrained domain. \Cref{tab:species} details each chemical model used and how much training data was generated by each model. 

\begin{table}[htb!]
    \centering
    \begin{tabular}{ccccc}
        \toprule 
         \textbf{Model} & \textbf{\# of Species} & \textbf{Reaction Steps} & \textbf{\# of Samples} & \textbf{Temperatures Simulated} (K) \\ \midrule 
         USC-II, \cite{wang_experimental_2022} & 111 & 784 & 16& \{450, 550\}\\ 
         Lu, \cite{lu_32_2017} & 32 & 206 & 80& \{450, 550,650,750,850\} \\ 
         Zettervall, \cite{zettervall_methodology_2021} & 23 & 66 & 160& \{450, 550,650,750,850\} \\ 
         AFRL, \cite{hassan_dynamic_2019} & 7 & 3 & 320& \{450, 550,650,750,850\} \\ 
         USAFA, \cite{hassan_error_2025} & 7 & 3 & 640& \{450, 550,650,750,850\} \\ \bottomrule 
    \end{tabular}
    \caption{Experimental details for multifidelity laminar flame speed calculations. }
    \label{tab:species}
\end{table}
Each of the models in \Cref{tab:species} was evaluated on equivalence ratio $(\phi)$ sweeps between 0.6 and 1.4 for a total number of samples indicated by the \textbf{\# of Samples} column. The most advanced model is the USC II which is considered to be the high-fidelity reference in this experiment. The high-fidelity training data are confined to a design space which is limited by temperature ($\leq 550$). \Cref{fig:lfs-objective} shows how the high-fidelity data varies with temperature and equivalence ratio. The black points represent the training data constrained by temperature and the red points represent unseen testing data on which each model is validated. 

\begin{figure}[htb!]
    \centering
    \includegraphics[width=1.0 \linewidth]{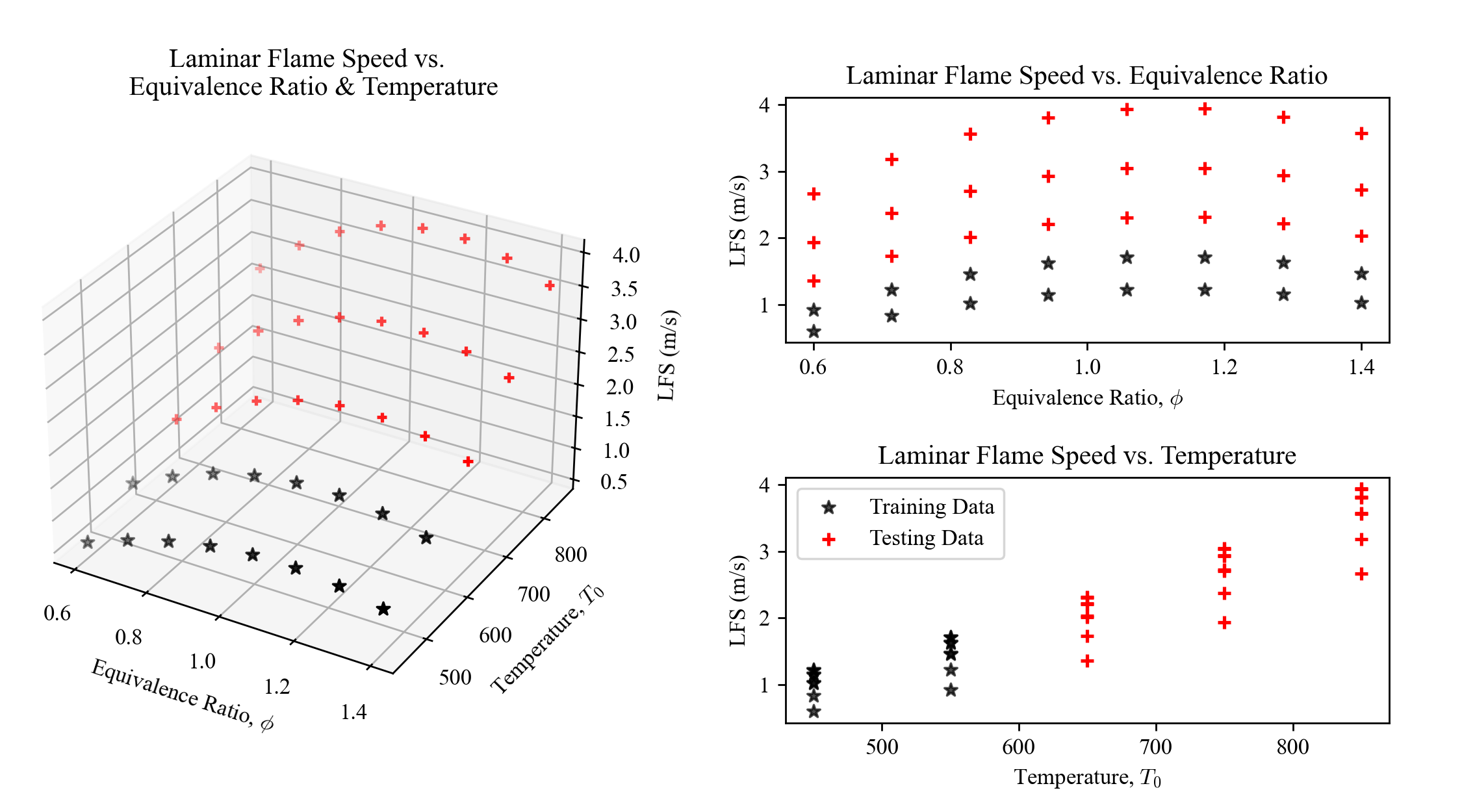}
    \caption{A visualization of the high-fidelity (USC II) training and testing data. The blue dots show the high-fidelity training data simulated at temperatures 450K and 550K. The black stars show the unseen testing data at temperatures 650K, 750K, and 850K. We emphasize how the behavior of the flame speed outside the design space (temperatures $\leq 550K$) differs significantly from the training data.}
    \label{fig:lfs-objective}
\end{figure}

As shown in \Cref{fig:lfs-objective}, the high-fidelity regression problem is fundamentally an extrapolation problem; we wish to predict high-fidelity model behavior outside of what is seen in the training data. The plots in \Cref{fig:lfs_plots} show the proposed method applied to the LFS data, compared with the Lu 206-step mechanism (the next-highest-fidelity computer model). \Cref{tab:lfs-results} contains the performance metrics of each method evaluated on the high-fidelity testing data. 

\begin{figure}[htb!]
    \centering
    \includegraphics[width=1.0\linewidth]{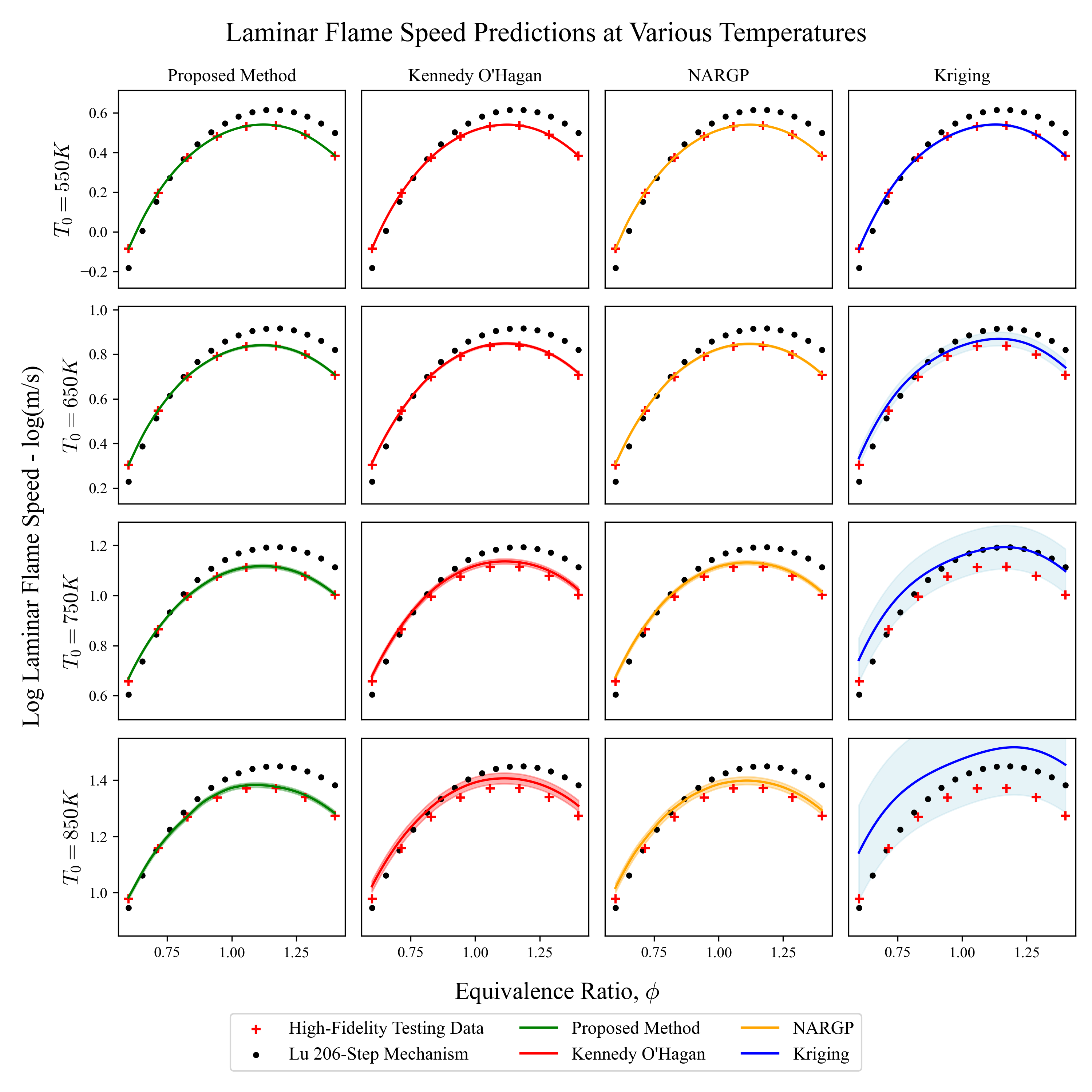}
    \caption{The laminar flame speed experiment captured at four different temperatures. The shaded regions represents a $\pm 2 \sigma$ confidence interval for the predictive posterior of each GP.}
    \label{fig:lfs_plots}
\end{figure}

\begin{table}[b!] 
    \centering
    \begin{tabular}{ccccc|c}\toprule 
         \textbf{Approach} & 550K* &  650K & 750K & 850K & \textbf{log ML} \\ \midrule 
        \textbf{Proposed Method}& \textbf{6.444e-03}  &  \textbf{6.176e-03}  & \textbf{6.984e-03} & \textbf{8.598e-03} & \textbf{82.8062} \\ 
        Kennedy O'Hagan     &             6.576e-03  &          8.442e-03 &          2.078e-02  &          3.636e-02 &          69.0738 \\ 
        NARGP               &             6.402e-03  &           7.453e-03 &           1.735e-02  &          2.990e-02 &          70.7028 \\ 
        Kriging             &             8.119e-03  &           3.381e-02 &          8.995e-02  &          1.680e-01 &          33.0798 \\ \bottomrule 
    \end{tabular}
    \caption{RMSE and log ML metrics for learned surrogate model predictions of laminar flame speed at various temperatures and equivalence ratios. *The temperature 550K is contained in the training data so we expect all models to perform comparably.}
    \label{tab:lfs-results}
\end{table}

The proposed method remains both more accurate in its mean predictions than the 206-step Lu mechanism and other multi-fidelity methods, but also contains the true flame speed values within its $\pm 2 \sigma$ confidence interval throughout the extrapolation process. At the furthest extrapolation, 850K, the proposed method achieves an RMSE 13x less than the next-best surrogate model (Kennedy O'Hagan). At 850K, the $\pm 2 \sigma$ confidence interval from single-fidelity Kriging is so wide that it is virtually unusable in estimating the true model. 

\subsection{Sparse Interpolation of Velocity Flow Field} \label{velocimetry}
Multifidelity machine learning can allow practitioners to accurately fill in gaps between sparse high-fidelity data. To show how the proposed method can improve the accuracy of such interpolation, we examine a supersonic computational fluid dynamics simulation (see \cite{peterson_overview_2024} for details). The multifidelity data was generated from two-dimensional Large Eddy Simulation (LES) and RANS (Reynolds-Averaged Navier Stokes) computer models evaluated at various spatial resolutions, outlined in \Cref{tab:cfd-sims}. In this scenario, the 125$\mu$m high-fidelity LES model is considered our high-fidelity model, as a stand-in for coarse experimental data such, as that proposed in \cite{pitz_chapter_2020}. The high-fidelity flow field contains large velocity gradients, including shock waves;  alternative approaches such as reduced-order modeling or data assimilation often have difficulty resolving these extreme features \cite{lucia_reduced_2001,zhou_neural_2026}. We seek to predict horizontal velocity (which dominates in magnitude) from x and y spatial coordinates. For high-fidelity training data, we selected a grid of points 5mm apart in the recirculation region (X coordinates $\leq 0.04$m) of the high-fidelity flow field (shown in \Cref{fig:sample-plot}). 

\begin{table}[htb!]
  \centering
  \begin{tabular}{cccc}
    \toprule
 \textbf{Simulation Technique} & \textbf{Resolution} & \textbf{\# of Cells} & \textbf{\# of Training Examples} \\
    \midrule
 LES  & 125\,$\mu$m & 116,021 & 45 \\
 LES  & 177\,$\mu$m & 58, 179 & 58,179\\
 LES  & 250\,$\mu$m & 29, 211 & 29, 211 \\
 LES  & 500\,$\mu$m & 7,406 & 7,406 \\
 RANS & 500\,$\mu$m & 7,406 & 7,406 \\
    \bottomrule
  \end{tabular}
  \caption{Experimental details for multifidelity velocity simulations.}
  \label{tab:cfd-sims}
\end{table}

Because of the plentiful low-fidelity training data, instead of using \glspl{gp} as surrogate models for the low-fidelity data, we used K-Nearest Neighbors (KNN) models, an efficient regression algorithm for large sets of training data and low-dimensional inputs. Training off-the-shelf KNN surrogate models for all four low-fidelity models took fewer than 5 seconds and just three lines of single-threaded code. In contrast, because the Kennedy O'Hagan and NARGP autoregressive models rely on GP modeling at every fidelity level, these frameworks were prohibitively expensive to implement due to the number of training data points available for the 177$\mu$m and 250$\mu$m simulations (58,179 and 29,211, respectively). To enable comparison with these frameworks, we used KNN approximations in place of trained low-fidelity \glspl{gp} for both the Kennedy O'Hagan and NARGP approaches as well as in our proposed approach. While this is not a 1:1 comparison, it underscores the need for efficient low-fidelity surrogate models in the presence of large volumes of training data. For the high-fidelity GPs, we used standard scaling of the inputs and low-fidelity model evaluations such that each entry of $\phi_1(\bx)$ had zero mean and unit variance. Performance metrics for each of the methods tested are displayed in \Cref{tab:velocimetry-results}.  A comparison of the KNN approximations of the low-fidelity models, the predictions of the proposed method compared to high-fidelity, and the predictions of single-fidelity GP regression/kriging compared to high-fidelity are shown in \Cref{fig:comparison-plot}. 

\begin{figure}[htb!]
    \centering
    \includegraphics[width=1.0\linewidth]{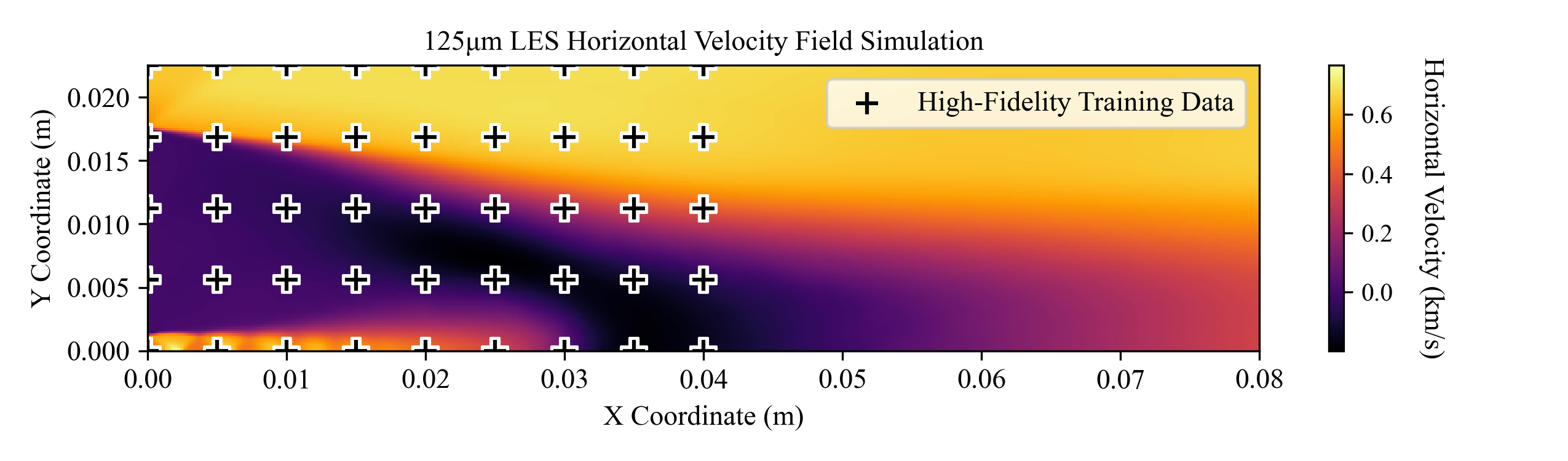}
    \caption{The full high-fidelity (125 $\mu$m LES) flow field with sparse high-fidelity training data indicated with the ``+'' symbols.}
    \label{fig:sample-plot}
\end{figure}

\begin{figure}[htb!]
    \centering
    \includegraphics[width=1.0\linewidth]{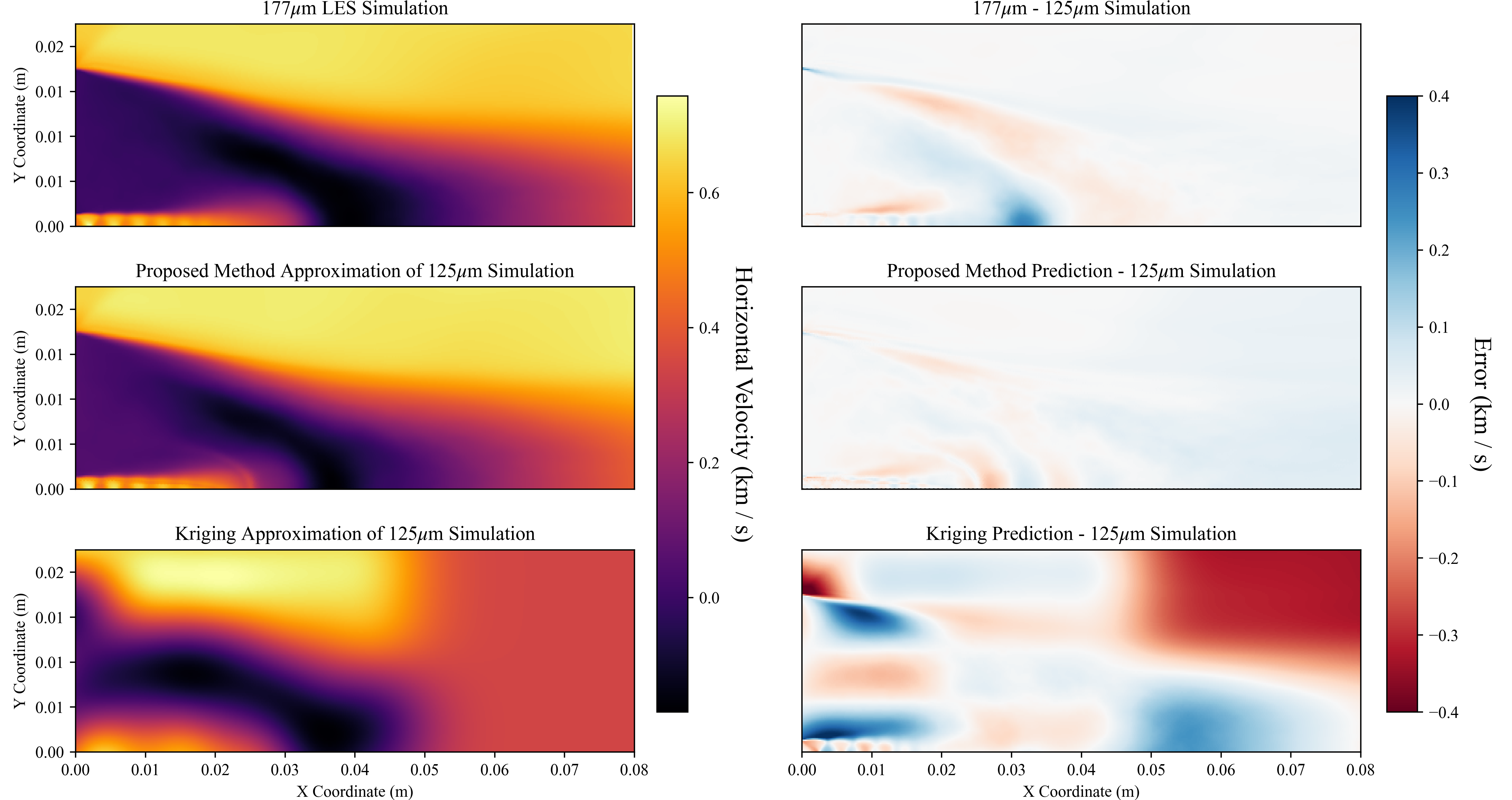}
    \caption{Results of the sparse flow-field interpolation experiment. The left half of the plots show approximations of high- and low-fidelity flow-fields. The right half of plots shows the error between the approximation to the target (obtained by subtracting the true simulation from the approximation). The first row compares the best low-fidelity simulation to the high-fidelity. The second row compares the proposed method's predictions with the true high-fidelity simulation. The last row compares single-fidelity kriging predictions with the true high-fidelity simulation.}
    \label{fig:comparison-plot}
\end{figure}

\begin{table}[htb!]
    \centering
    \begin{tabular}{ccccc}
         \toprule 
         \textbf{Appoach} & \textbf{RMSE} & $\mathbf{R^2}$ & \textbf{log ML}  \\
         \midrule 
Proposed Method        &  \textbf{2.1479e-02}  &  \textbf{0.9971}  &  \textbf{133.2262} \\
Kennedy OH             &  2.5785e-02  &  0.9951  &  90.2265 \\
NARGP                  &  6.0544e-02  &  0.9897  &  104.3015 \\
Kriging                &  1.3079e-01  &  0.9112  &  32.6717 \\
177$\mu$m             &  4.1232e-02  &  0.9843  &  -- \\ 
250$\mu$m             &  7.8781e-02  &  0.9418  &  -- \\ 
500$\mu$m             &  1.1332e-01  &  0.8821  &  -- \\ 
RANS$\mu$m            &  6.7716e-02  &  0.9569  &  -- \\  \bottomrule
    \end{tabular}
    \caption{Performance metrics for the sparse flow-field interpolation experiment comparing learned high-fidelity surrogate models and low-fidelity CFD models to the true high-fidelity flow field.}
    \label{tab:velocimetry-results}
\end{table}

We emphasize that single-fidelity \gls{gp} regression/kriging, achieving the worst RMSE of all models on the testing data, fully interpolates its training data; with scarce data, low training error provides little guarantee of a reliable surrogate model. The proposed method achieves better performance metrics than the existing multifidelity methods tested and all low-fidelity computer models. This experiment demonstrates the ability of learned multifidelity surrogate models to significantly outperform lower-resolution CFD simulations in the presence of scarce high-fidelity data. We also demonstrate that non-GP regression methods (e.g., KNN) for the low-fidelity surrogate models enable vastly cheaper training, as evidenced by our hardware's inability to train unapproximated low-fidelity \glspl{gp} in this problem. 

Finally, one key advantage of GPs over alternative regression methods is their estimation of predictive uncertainty. \Cref{fig:validation-plot} shows two plots: one showing the difference between the surrogate model's mean function and the true function, the other showing the standard deviation of the surrogate model's predictive posterior distribution. Even on unseen testing data, the predictive uncertainty (standard deviation) of the learned surrogate model is over 65\% correlated with the true error over the entire flow field. This uncertainty estimation provides an additional layer of trust in the model's online predictions. However, we emphasize that GP regression is only accurate if the true model belongs to the hypothesis class specified by the kernel and mean functions. Physical constraints, such as positivity and conservation laws, may be violated by GP-based methods; thus, restricting GP priors to only include physically feasible functions can further increase the credibility of model uncertainty estimates. A 99\% confidence interval of the proposed method's predictions contains the true flow field only 92.4\% of the time; this is evidence of slight model misspecification which could be remedied by more physically-informed kernels and mean functions, which we will leave to future work. 

\begin{figure}[htb!]
    \centering
    \includegraphics[width=1.0\linewidth]{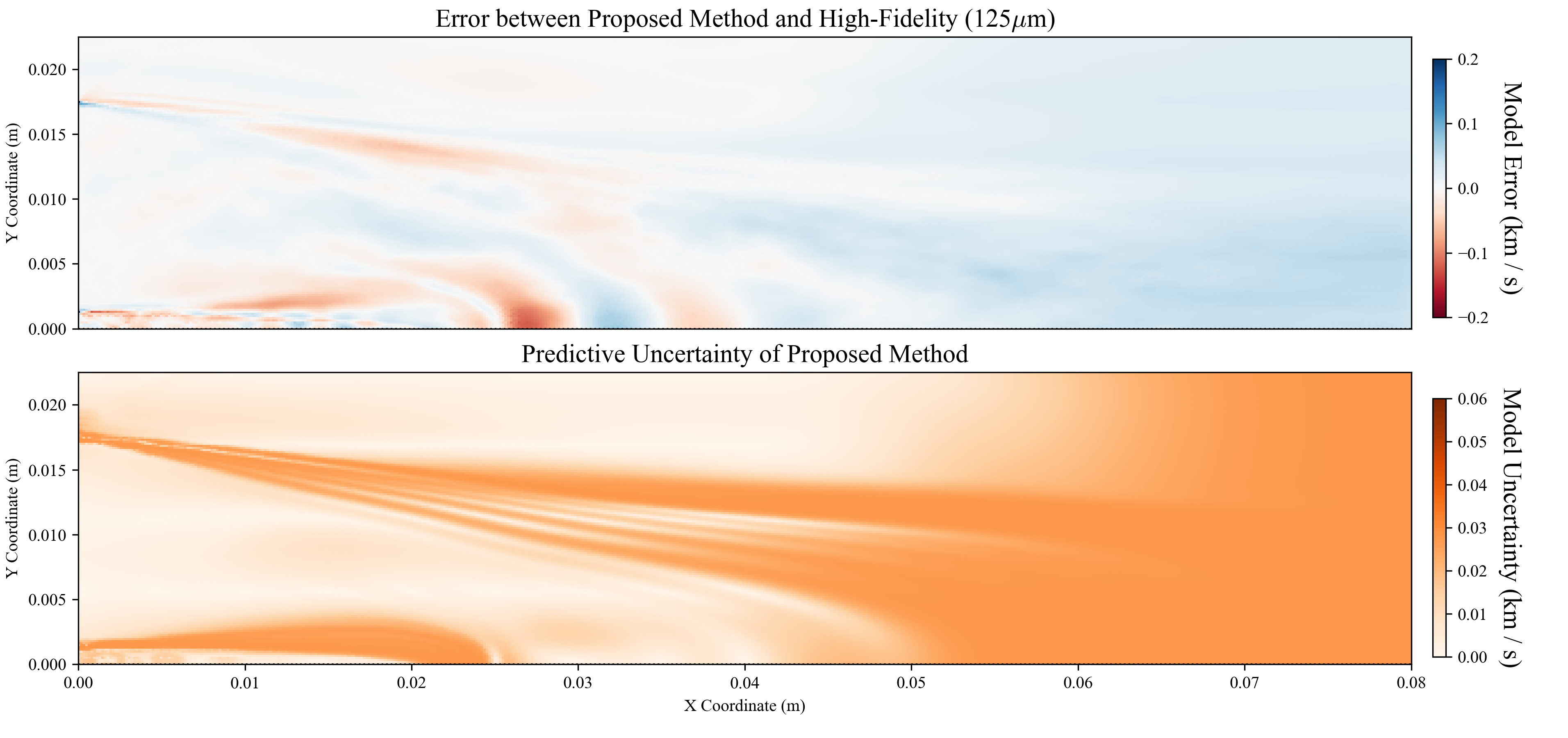}
    \caption{Comparison between model error and predictive uncertainty. The top plot shows the difference between the learned model's predictive mean and the true flow field (blue regions represent over-prediction and red regions represent under-prediction). The bottom plot displays the standard deviation of the surrogate model's Gaussian posterior predictive distribution, the square-root of the variance in \cref{eqn:gpr-posterior}. The model uncertainty is roughly $65$\% correlated with absolute error.}
    \label{fig:validation-plot}
\end{figure}

\section{Conclusion \& Discussion} \label{conclusion}

In this work, we propose a multifidelity machine learning approach which improves upon existing single- and multifidelity \gls{gp}-based methods. This improvement is achieved by removing several limitations of existing estimators, specifically the first-order autoregressive structure, the assumption of noiseless and nested training data, the assumption of linear relationships between levels of fidelity, and the requirement of \glspl{gp} as low-fidelity surrogate models. This allows nonlinear combinations of all low-fidelity surrogate models to influence the prediction of the high-fidelity model. We demonstrate improvements over single-fidelity kriging, Kennedy O'Hagan, and NARGP estimators in various accuracy metrics on three benchmark problems. The proposed method shows an ability to extrapolate outside of a constrained high-fidelity design space on a chemical kinetics example and to interpolate a highly nonlinear velocity field from scarce training data.  

Future work utilizing this framework may take several meaningful directions. First, the many existing methods which modify the \gls{gp} formulation in \cite{rasmussen_gaussian_2008} may be used to improve the accuracy and scalability of the high-fidelity surrogate model (e.g., deep kernel methods, sparse approximations, deep GPs, hierarchical Bayesian approaches, non-Gaussian assumptions, etc.). As seen in \Cref{fig:synthetic-results,fig:validation-plot}, the proposed method produces undesired oscillations in its predictive posterior; further investigation into mitigating such artifacts may result in more accurate estimators. Moreover, the proposed approach does not incorporate any physical constraints in training the high-fidelity surrogate model. One way to increase the credibility of the predictive posterior produced by the proposed method is by restricting the high-fidelity GP prior to physically feasible candidates. Additionally, exploration into how to optimally preprocess the low-fidelity model evaluations as inputs to the kernel may improve generalization and accelerate convergence to optimal hyperparameters. Specifically, the ordering of the low-fidelity models affects the predictions of both the proposed method and the existing methods; investigation into optimal low-fidelity model ordering schemes may yield more accurate high-fidelity estimates. Lastly, investigation into the tradeoff between offline training cost and model accuracy for different low-fidelity surrogate models may lead to effective active learning and experimental design schemes across all levels of fidelity.

\section{Acknowledgments} 
Atticus Rex was supported in parts by the National Science Foundation under Grant No. DGE-2039655 and by the Department of Defense (DoD) High Performance Computing Modernization Program in collaboration with an appointment to the DoD Research Participation Program administered by the Oak Ridge Institute for Science and Education (ORISE) through an interagency agreement between the U.S. Department of Energy (DOE) and the DoD. ORISE is managed by ORAU under DOE contract number DE-SC0014664. Elizabeth Qian was supported in parts by the US Department of Energy Office of Science under DE-SC0024721.
Both authors acknowledge partial support by the National Science Foundation under CMMI-2442140. Any opinions, findings, and conclusions or recommendations expressed in this material are those of the authors and do not necessarily reflect the views of the NSF, DoD, DOE, or ORAU/ORISE. 

\newpage 

\appendix 

\section{Algorithmic Complexities for Alternative Regression Models} \label{alternative-algo-costs}

\begin{table}[htbp!]
    \centering
    \begin{tabular}{ccccc}
         \toprule \textbf{Method}&  \textbf{Training Time } & \textbf{Training Space} & \textbf{Prediction Time } & \textbf{Prediction Space}\\ \midrule 
         Kriging/GP & $\setO(kN^3) $ & $\setO(N^2)$ & $\setO(N)$ & $\setO(N)$\\ 
         K-Nearest Neighbors & $\setO(Nd)$ & $\setO(Nd)$ & $\setO(Nd)$ & $\setO(Nd)$ \\ 
         Deep Neural Network & $\setO(kNd) $ & $\setO(Nd)$ & $\setO(1)$ & $\setO(1)$ \\ 
         OLS/Ridge Regression & $\setO(Nd^2 + d^3)$ & $\setO(Nd + d^2)$ & $\setO(d)$ & $\setO(d)$ \\ 
         Lasso Regression & $\setO(kNd)$ & $\setO(Nd)$ &$\setO(d)$ & $\setO(d)$ \\ 
         Polynomial Regression & $\setO(N d'^2 + d'^3)$ & $\setO(N d')$ & $\setO(d')$ & $\setO(d')$\\ 
         Decision Tree & $\setO(N d \log N)$ & $\setO(Nd)$ & $\setO(\log N)$ & $\setO(N)$\\ 
         Random Forest & $\setO(N d \log N)$ & $\setO(N)$ & $\setO(\log N)$ & $\setO(N)$\\ 
         \bottomrule 
    \end{tabular}
    \caption{Offline training and online prediction algorithmic costs for different regression models with respect to number of training data points $N$, input dimension $d$, polynomial input dimension $d'$, and number of gradient descent steps $k$. We note that for some of these expressions, important model parameters such as number of trees in Random Forest and hidden layer count/dimension for Neural Networks have been omitted. We emphasize that all non-GP methods incur a sub-cubic training time complexity and a sub-quadratic training space complexity with respect to the number of training data points. }
    \label{tab:algo-comparison}
\end{table}

\printbibliography[heading=bibintoc]

@article{yang_when_2020,
	title = {When {Bifidelity} {Meets} {CoKriging}: {An} {Efficient} {Physics}-{Informed} {MultiFidelity} {Method}},
	volume = {42},
	issn = {1064-8275, 1095-7197},
	shorttitle = {When {Bifidelity} {Meets} {CoKriging}},
	url = {https://epubs.siam.org/doi/10.1137/18M1231353},
	doi = {10.1137/18M1231353},
	language = {en},
	number = {1},
	urldate = {2026-06-23},
	journal = {SIAM Journal on Scientific Computing},
	author = {Yang, Xiu and Zhu, Xueyu and Li, Jing},
	month = jan,
	year = {2020},
	pages = {A220--A249},
}

@article{yang_physics-informed_2019,
	title = {Physics-{Informed} {CoKriging}: {A} {Gaussian}-{Process}-{Regression}-{Based} {Multifidelity} {Method} for {Data}-{Model} {Convergence}},
	volume = {395},
	issn = {00219991},
	shorttitle = {Physics-{Informed} {CoKriging}},
	url = {http://arxiv.org/abs/1811.09757},
	doi = {10.1016/j.jcp.2019.06.041},
	urldate = {2026-06-23},
	journal = {Journal of Computational Physics},
	author = {Yang, Xiu and Barajas-Solano, David and Tartakovsky, Guzel and Tartakovsky, Alexandre},
	month = oct,
	year = {2019},
	note = {arXiv:1811.09757 [stat.ML]},
	pages = {410--431},
}

@article{spitieris_bayesian_2023,
	title = {Bayesian {Calibration} of {Imperfect} {Computer} {Models} using {Physics}-{Informed} {Priors}},
	volume = {24},
	issn = {1533-7928},
	url = {http://jmlr.org/papers/v24/22-0676.html},
	number = {108},
	urldate = {2026-06-23},
	journal = {Journal of Machine Learning Research},
	author = {Spitieris, Michail and Steinsland, Ingelin},
	year = {2023},
	pages = {1--39},
}

@article{raissi_numerical_2018,
	title = {Numerical {Gaussian} {Processes} for {Time}-{Dependent} and {Nonlinear} {Partial} {Differential} {Equations}},
	volume = {40},
	issn = {1064-8275},
	url = {https://epubs.siam.org/doi/10.1137/17M1120762},
	doi = {10.1137/17M1120762},
	number = {1},
	urldate = {2026-06-13},
	journal = {SIAM Journal on Scientific Computing},
	publisher = {Society for Industrial and Applied Mathematics},
	author = {Raissi, Maziar and Perdikaris, Paris and Karniadakis, George Em},
	month = jan,
	year = {2018},
	pages = {A172--A198},
}

@article{raissi_machine_2017,
	title = {Machine {Learning} of {Linear} {Differential} {Equations} using {Gaussian} {Processes}},
	volume = {348},
	issn = {00219991},
	url = {http://arxiv.org/abs/1701.02440},
	doi = {10.1016/j.jcp.2017.07.050},
	urldate = {2026-06-13},
	journal = {Journal of Computational Physics},
	author = {Raissi, Maziar and Karniadakis, George Em},
	month = nov,
	year = {2017},
	note = {arXiv:1701.02440 [cs.LG]},
	pages = {683--693},
}

@book{zettervall_methodology_2021,
	address = {Lund},
	title = {Methodology for developing reduced reaction mechanisms, and their use in combustion simulations},
	isbn = {978-91-7895-729-3},
	language = {en},
	publisher = {Division of Combustion Physics, Lund University},
	author = {Zettervall, Niklas},
	year = {2021},
}

@book{neal_bayesian_1996,
	address = {New York, NY},
	series = {Lecture {Notes} in {Statistics}},
	title = {Bayesian {Learning} for {Neural} {Networks}},
	volume = {118},
	copyright = {http://www.springer.com/tdm},
	isbn = {978-0-387-94724-2 978-1-4612-0745-0},
	url = {http://link.springer.com/10.1007/978-1-4612-0745-0},
	doi = {10.1007/978-1-4612-0745-0},
	urldate = {2026-03-16},
	publisher = {Springer},
	author = {Neal, Radford M.},
	editor = {Bickel, P. and Diggle, P. and Fienberg, S. and Krickeberg, K. and Olkin, I. and Wermuth, N. and Zeger, S.},
	year = {1996},
}

@article{stein_universal_1991,
	title = {Universal {Kriging} and {Cokriging} as a {Regression} {Procedure}},
	volume = {47},
	issn = {0006-341X},
	url = {https://www.jstor.org/stable/2532147},
	doi = {10.2307/2532147},
	number = {2},
	urldate = {2026-03-16},
	journal = {Biometrics},
	publisher = {International Biometric Society},
	author = {Stein, A. and Corsten, L. C. A.},
	year = {1991},
	pages = {575--587},
}

@incollection{garland_multi-fidelity_2020,
	address = {Cham},
	title = {Multi-{Fidelity} for {MDO} {Using} {Gaussian} {Processes}},
	isbn = {978-3-030-39126-3},
	url = {https://doi.org/10.1007/978-3-030-39126-3_8},
	doi = {10.1007/978-3-030-39126-3_8},
	language = {en},
	urldate = {2026-03-05},
	booktitle = {Aerospace {System} {Analysis} and {Optimization} in {Uncertainty}},
	publisher = {Springer International Publishing},
	author = {Garland, Nicolas and Le Riche, Rodolphe and Richet, Yann and Durrande, Nicolas},
	editor = {Brevault, Loïc and Balesdent, Mathieu and Morio, Jérôme},
	year = {2020},
	pages = {295--320},
}

@article{christen_markov_2005,
	title = {Markov chain {Monte} {Carlo} {Using} an {Approximation}},
	volume = {14},
	issn = {1061-8600},
	url = {https://doi.org/10.1198/106186005X76983},
	doi = {10.1198/106186005X76983},
	number = {4},
	urldate = {2026-03-05},
	journal = {Journal of Computational and Graphical Statistics},
	publisher = {Taylor \& Francis},
	author = {Christen, J. Andrés and Fox, Colin},
	month = dec,
	year = {2005},
	note = {\_eprint: https://doi.org/10.1198/106186005X76983},
	pages = {795--810},
}

@misc{lu_32_2017,
	title = {A 32 species skeletal mechanism developped by {Dr} {T}. {Lu} \& co-workers},
	url = {https://www.cerfacs.fr/cantera/mechanisms/eth.php},
	urldate = {2025-07-22},
	author = {Lu, T},
	year = {2017},
}

@misc{zhou_neural_2026,
	title = {Neural ensemble {Kalman} filter: {Data} assimilation for compressible flows with shocks},
	shorttitle = {Neural ensemble {Kalman} filter},
	url = {http://arxiv.org/abs/2602.23461},
	doi = {10.48550/arXiv.2602.23461},
	urldate = {2026-03-03},
	publisher = {arXiv},
	author = {Zhou, Xu-Hui and Beronilla, Lorenzo and Sleeman, Michael K. and Hu, Hangchuan and Morzfeld, Matthias and Stuart, Andrew M. and Zaki, Tamer A.},
	month = feb,
	year = {2026},
	note = {arXiv:2602.23461 [physics]},
}

@phdthesis{lucia_reduced_2001,
	type = {Ph.{D}. thesis},
	title = {Reduced order modeling for high-speed flows with moving shocks},
	url = {https://ui.adsabs.harvard.edu/abs/2001PhDT.......167L},
	urldate = {2026-03-03},
	author = {Lucia, David Jerry},
	month = jan,
	year = {2001},
	note = {ADS Bibcode: 2001PhDT.......167L},
}

@incollection{hassan_error_2025,
	title = {Error {In} {Shock} {Tube} {Ignition} {Delay} {Time} {Predictions} {Due} to {Bifurcation} and {Boundary} {Layer} {Effects}},
	url = {https://arc.aiaa.org/doi/abs/10.2514/6.2025-2140},
	doi = {10.2514/6.2025-2140},
	urldate = {2026-03-03},
	booktitle = {{AIAA} {SCITECH} 2025 {Forum}},
	publisher = {American Institute of Aeronautics and Astronautics},
	author = {Hassan, Ez A. and Hageman, Mitchell D. and Knadler, Michael S. and Lucas, Stephen},
	month = mar,
	year = {2025},
	note = {\_eprint: https://arc.aiaa.org/doi/pdf/10.2514/6.2025-2140},
}

@article{hassan_dynamic_2019,
	title = {Dynamic {Hybrid} {Reynolds}-{Averaged} {Navier}–{Stokes}/{Large}-{Eddy} {Simulation} of a {Supersonic} {Cavity}: {Chemistry} {Effects}},
	volume = {35},
	issn = {0748-4658},
	shorttitle = {Dynamic {Hybrid} {Reynolds}-{Averaged} {Navier}–{Stokes}/{Large}-{Eddy} {Simulation} of a {Supersonic} {Cavity}},
	url = {https://arc.aiaa.org/doi/10.2514/1.B37092},
	doi = {10.2514/1.B37092},
	number = {1},
	urldate = {2026-03-03},
	journal = {Journal of Propulsion and Power},
	publisher = {American Institute of Aeronautics and Astronautics},
	author = {Hassan, Ez and Peterson, David M. and Walters, D. Keith and Luke, Edward A.},
	month = jan,
	year = {2019},
	pages = {201--212},
}

@article{yano_optimization-based_2012,
	title = {An optimization-based framework for anisotropic simplex mesh adaptation},
	volume = {231},
	issn = {0021-9991},
	url = {https://www.sciencedirect.com/science/article/pii/S0021999112003749},
	doi = {10.1016/j.jcp.2012.06.040},
	number = {22},
	urldate = {2026-02-12},
	journal = {Journal of Computational Physics},
	author = {Yano, Masayuki and Darmofal, David L.},
	month = sep,
	year = {2012},
	pages = {7626--7649},
}

@article{amsallem_design_2015,
	title = {Design optimization using hyper-reduced-order models},
	volume = {51},
	issn = {1615-147X, 1615-1488},
	url = {http://link.springer.com/10.1007/s00158-014-1183-y},
	doi = {10.1007/s00158-014-1183-y},
	language = {en},
	number = {4},
	urldate = {2026-02-12},
	journal = {Structural and Multidisciplinary Optimization},
	author = {Amsallem, David and Zahr, Matthew and Choi, Youngsoo and Farhat, Charbel},
	month = apr,
	year = {2015},
	pages = {919--940},
}

@article{zahr_multilevel_2017,
	title = {A multilevel projection‐based model order reduction framework for nonlinear dynamic multiscale problems in structural and solid mechanics},
	volume = {112},
	copyright = {http://onlinelibrary.wiley.com/termsAndConditions\#vor},
	issn = {0029-5981, 1097-0207},
	url = {https://onlinelibrary.wiley.com/doi/10.1002/nme.5535},
	doi = {10.1002/nme.5535},
	language = {en},
	number = {8},
	urldate = {2026-02-12},
	journal = {International Journal for Numerical Methods in Engineering},
	author = {Zahr, Matthew J. and Avery, Philip and Farhat, Charbel},
	month = nov,
	year = {2017},
	pages = {855--881},
}

@inproceedings{song_general_2019,
	title = {A {General} {Framework} for {Multi}-fidelity {Bayesian} {Optimization} with {Gaussian} {Processes}},
	issn = {2640-3498},
	url = {https://proceedings.mlr.press/v89/song19b.html},
	language = {en},
	urldate = {2026-02-12},
	booktitle = {Proceedings of the {Twenty}-{Second} {International} {Conference} on {Artificial} {Intelligence} and {Statistics}},
	publisher = {PMLR},
	author = {Song, Jialin and Chen, Yuxin and Yue, Yisong},
	month = apr,
	year = {2019},
	pages = {3158--3167},
}

@misc{li_multi-fidelity_2024,
	title = {Multi-{Fidelity} {Methods} for {Optimization}: {A} {Survey}},
	shorttitle = {Multi-{Fidelity} {Methods} for {Optimization}},
	url = {http://arxiv.org/abs/2402.09638},
	doi = {10.48550/arXiv.2402.09638},
	urldate = {2026-02-12},
	publisher = {arXiv},
	author = {Li, Ke and Li, Fan},
	month = feb,
	year = {2024},
	note = {arXiv:2402.09638 [cs]},
}

@article{forrester_multi-fidelity_2007,
	title = {Multi-fidelity optimization via surrogate modelling},
	volume = {463},
	issn = {1364-5021},
	url = {https://doi.org/10.1098/rspa.2007.1900},
	doi = {10.1098/rspa.2007.1900},
	number = {2088},
	urldate = {2026-02-12},
	journal = {Proceedings of the Royal Society A: Mathematical, Physical and Engineering Sciences},
	author = {Forrester, Alexander I.J and Sóbester, András and Keane, Andy J},
	month = oct,
	year = {2007},
	pages = {3251--3269},
}

@article{keil_relaxed_2024,
	title = {A relaxed localized trust-region reduced basis approach for optimization of multiscale problems},
	volume = {58},
	copyright = {© The authors. Published by EDP Sciences, SMAI 2024},
	issn = {2822-7840, 2804-7214},
	url = {https://www.esaim-m2an.org/articles/m2an/abs/2024/01/m2an220169/m2an220169.html},
	doi = {10.1051/m2an/2023089},
	language = {en},
	number = {1},
	urldate = {2026-02-12},
	journal = {ESAIM: Mathematical Modelling and Numerical Analysis},
	publisher = {EDP Sciences},
	author = {Keil, Tim and Ohlberger, Mario},
	month = jan,
	year = {2024},
	pages = {79--105},
}

@misc{klein_multi-fidelity_2025,
	title = {Multi-fidelity {Learning} of {Reduced} {Order} {Models} for {Parabolic} {PDE} {Constrained} {Optimization}},
	url = {http://arxiv.org/abs/2503.21252},
	doi = {10.48550/arXiv.2503.21252},
	urldate = {2026-02-12},
	publisher = {arXiv},
	author = {Klein, Benedikt and Ohlberger, Mario},
	month = mar,
	year = {2025},
	note = {arXiv:2503.21252 [math]},
}

@article{qian_certified_2017,
	title = {A {Certified} {Trust} {Region} {Reduced} {Basis} {Approach} to {PDE}-{Constrained} {Optimization}},
	volume = {39},
	issn = {1064-8275},
	url = {https://epubs.siam.org/doi/abs/10.1137/16M1081981},
	doi = {10.1137/16M1081981},
	number = {5},
	urldate = {2026-02-12},
	journal = {SIAM Journal on Scientific Computing},
	publisher = {Society for Industrial and Applied Mathematics},
	author = {Qian, Elizabeth and Grepl, Martin and Veroy, Karen and Willcox, Karen},
	month = jan,
	year = {2017},
	pages = {S434--S460},
}

@article{qian_multifidelity_2018,
	title = {Multifidelity {Monte} {Carlo} {Estimation} of {Variance} and {Sensitivity} {Indices}},
	volume = {6},
	url = {https://epubs.siam.org/doi/abs/10.1137/17M1151006},
	doi = {10.1137/17M1151006},
	number = {2},
	urldate = {2026-02-12},
	journal = {SIAM/ASA Journal on Uncertainty Quantification},
	publisher = {Society for Industrial and Applied Mathematics},
	author = {Qian, E. and Peherstorfer, B. and O'Malley, D. and Vesselinov, V. V. and Willcox, K.},
	month = jan,
	year = {2018},
	pages = {683--706},
}

@article{lu_multifidelity_2022,
	title = {Multifidelity deep neural operators for efficient learning of partial differential equations with application to fast inverse design of nanoscale heat transport},
	volume = {4},
	url = {https://link.aps.org/doi/10.1103/PhysRevResearch.4.023210},
	doi = {10.1103/PhysRevResearch.4.023210},
	number = {2},
	urldate = {2026-02-12},
	journal = {Physical Review Research},
	publisher = {American Physical Society},
	author = {Lu, Lu and Pestourie, Raphaël and Johnson, Steven G. and Romano, Giuseppe},
	month = jun,
	year = {2022},
	pages = {023210},
}

@article{howard_multifidelity_2024,
	title = {A multifidelity approach to continual learning for physical systems},
	volume = {5},
	issn = {2632-2153},
	url = {https://doi.org/10.1088/2632-2153/ad45b2},
	doi = {10.1088/2632-2153/ad45b2},
	language = {en},
	number = {2},
	urldate = {2026-02-12},
	journal = {Machine Learning: Science and Technology},
	publisher = {IOP Publishing},
	author = {Howard, Amanda and Fu, Yucheng and Stinis, Panos},
	month = may,
	year = {2024},
	pages = {025042},
}

@misc{heinlein_multifidelity_2024,
	title = {Multifidelity domain decomposition-based physics-informed neural networks and operators for time-dependent problems},
	url = {http://arxiv.org/abs/2401.07888},
	doi = {10.48550/arXiv.2401.07888},
	language = {en},
	urldate = {2026-02-12},
	publisher = {arXiv},
	author = {Heinlein, Alexander and Howard, Amanda A. and Beecroft, Damien and Stinis, Panos},
	month = jun,
	year = {2024},
	note = {arXiv:2401.07888 [math]},
}

@article{howard_multifidelity_2023,
	title = {Multifidelity deep operator networks for data-driven and physics-informed problems},
	volume = {493},
	issn = {0021-9991},
	url = {https://www.sciencedirect.com/science/article/pii/S0021999123005570},
	doi = {10.1016/j.jcp.2023.112462},
	urldate = {2026-02-12},
	journal = {Journal of Computational Physics},
	author = {Howard, Amanda A. and Perego, Mauro and Karniadakis, George Em and Stinis, Panos},
	month = nov,
	year = {2023},
	pages = {112462},
}

@article{popov_multifidelity_2021,
	title = {A {Multifidelity} {Ensemble} {Kalman} {Filter} with {Reduced} {Order} {Control} {Variates}},
	volume = {43},
	issn = {1064-8275},
	url = {https://epubs.siam.org/doi/abs/10.1137/20M1349965},
	doi = {10.1137/20M1349965},
	number = {2},
	urldate = {2026-01-26},
	journal = {SIAM Journal on Scientific Computing},
	publisher = {Society for Industrial and Applied Mathematics},
	author = {Popov, Andrey A. and Mou, Changhong and Sandu, Adrian and Iliescu, Traian},
	month = jan,
	year = {2021},
	pages = {A1134--A1162},
}

@article{warne_multifidelity_2022,
	title = {Multifidelity multilevel {Monte} {Carlo} to accelerate approximate {Bayesian} parameter inference for partially observed stochastic processes},
	volume = {469},
	issn = {0021-9991},
	url = {https://www.sciencedirect.com/science/article/pii/S0021999122006052},
	doi = {10.1016/j.jcp.2022.111543},
	urldate = {2026-01-17},
	journal = {Journal of Computational Physics},
	author = {Warne, David J. and Prescott, Thomas P. and Baker, Ruth E. and Simpson, Matthew J.},
	month = nov,
	year = {2022},
	pages = {111543},
}

@article{catanach_bayesian_2020,
	title = {Bayesian inference of {Stochastic} reaction networks using {Multifidelity} {Sequential} {Tempered} {Markov} {Chain} {Monte} {Carlo}},
	volume = {10},
	issn = {2152-5080},
	url = {https://www.osti.gov/biblio/1670752},
	doi = {10.1615/int.j.uncertaintyquantification.2020033241},
	language = {English},
	number = {6},
	urldate = {2026-01-17},
	journal = {International Journal for Uncertainty Quantification},
	publisher = {Begell House},
	author = {Catanach, Thomas A. and Vo, Huy D. and Munsky, Brian},
	month = jun,
	year = {2020},
	note = {Number: SAND2020--7898J; SAND2019--15382J},
}

@article{cai_multi-fidelity_2022,
	title = {Multi-fidelity {Monte} {Carlo}: a pseudo-marginal approach},
	volume = {35},
	shorttitle = {Multi-fidelity {Monte} {Carlo}},
	url = {https://proceedings.neurips.cc/paper_files/paper/2022/hash/8803b9ae0b13011f28e6dd57da2ebbd8-Abstract-Conference.html?utm_source=chatgpt.com},
	language = {en},
	urldate = {2026-01-17},
	journal = {Advances in Neural Information Processing Systems},
	author = {Cai, Diana and Adams, Ryan P.},
	month = dec,
	year = {2022},
	pages = {21654--21667},
}

@misc{peherstorfer_transport-based_2018,
	title = {A transport-based multifidelity preconditioner for {Markov} chain {Monte} {Carlo}},
	url = {http://arxiv.org/abs/1808.09379},
	doi = {10.48550/arXiv.1808.09379},
	urldate = {2026-01-17},
	publisher = {arXiv},
	author = {Peherstorfer, Benjamin and Marzouk, Youssef},
	month = aug,
	year = {2018},
	note = {arXiv:1808.09379 [math]},
}

@article{giles_multilevel_2015,
	title = {Multilevel {Monte} {Carlo} methods},
	volume = {24},
	issn = {0962-4929, 1474-0508},
	url = {https://www.cambridge.org/core/journals/acta-numerica/article/abs/multilevel-monte-carlo-methods/C5AF9A57ED8FF8FDF08074C1071C5511},
	doi = {10.1017/S096249291500001X},
	language = {en},
	urldate = {2026-01-17},
	journal = {Acta Numerica},
	author = {Giles, Michael B.},
	month = may,
	year = {2015},
	pages = {259--328},
}

@misc{gorodetsky_grouped_2024,
	title = {Grouped approximate control variate estimators},
	url = {http://arxiv.org/abs/2402.14736},
	doi = {10.48550/arXiv.2402.14736},
	urldate = {2026-01-12},
	publisher = {arXiv},
	author = {Gorodetsky, Alex A. and Jakeman, John D. and Eldred, Michael S.},
	month = feb,
	year = {2024},
	note = {arXiv:2402.14736 [stat]},
}

@inproceedings{snoek_practical_2012,
	title = {Practical {Bayesian} {Optimization} of {Machine} {Learning} {Algorithms}},
	volume = {25},
	url = {https://papers.nips.cc/paper_files/paper/2012/hash/05311655a15b75fab86956663e1819cd-Abstract.html},
	urldate = {2025-11-03},
	booktitle = {Advances in {Neural} {Information} {Processing} {Systems}},
	publisher = {Curran Associates, Inc.},
	author = {Snoek, Jasper and Larochelle, Hugo and Adams, Ryan P},
	year = {2012},
}

@incollection{berger_bayesian_1985,
	address = {New York, NY},
	title = {Bayesian {Analysis}},
	isbn = {978-1-4757-4286-2},
	url = {https://doi.org/10.1007/978-1-4757-4286-2_4},
	doi = {10.1007/978-1-4757-4286-2_4},
	language = {en},
	urldate = {2025-11-03},
	booktitle = {Statistical {Decision} {Theory} and {Bayesian} {Analysis}},
	publisher = {Springer},
	author = {Berger, James O.},
	editor = {Berger, James O.},
	year = {1985},
	pages = {118--307},
}

@incollection{neal_priors_1996,
	address = {New York, NY},
	title = {Priors for {Infinite} {Networks}},
	isbn = {978-1-4612-0745-0},
	url = {https://doi.org/10.1007/978-1-4612-0745-0_2},
	doi = {10.1007/978-1-4612-0745-0_2},
	language = {en},
	urldate = {2025-11-03},
	booktitle = {Bayesian {Learning} for {Neural} {Networks}},
	publisher = {Springer},
	author = {Neal, Radford M.},
	editor = {Neal, Radford M.},
	year = {1996},
	pages = {29--53},
}

@article{hornik_multilayer_1989,
	title = {Multilayer feedforward networks are universal approximators},
	volume = {2},
	issn = {0893-6080},
	url = {https://www.sciencedirect.com/science/article/pii/0893608089900208},
	doi = {10.1016/0893-6080(89)90020-8},
	number = {5},
	urldate = {2025-11-03},
	journal = {Neural Networks},
	author = {Hornik, Kurt and Stinchcombe, Maxwell and White, Halbert},
	month = jan,
	year = {1989},
	pages = {359--366},
}

@book{hastie_elements_2009,
	title = {The {Elements} of {Statistical} {Learning}: {Data} {Mining}, {Inference}, and {Prediction}, {Second} {Edition}},
	isbn = {978-0-387-84858-7},
	shorttitle = {The {Elements} of {Statistical} {Learning}},
	language = {en},
	publisher = {Springer Science \& Business Media},
	author = {Hastie, Trevor and Tibshirani, Robert and Friedman, Jerome},
	month = aug,
	year = {2009},
	note = {Google-Books-ID: tVIjmNS3Ob8C},
}

@article{vapnik_overview_1999,
	title = {An overview of statistical learning theory},
	volume = {10},
	issn = {1045-9227},
	doi = {10.1109/72.788640},
	language = {eng},
	number = {5},
	journal = {IEEE transactions on neural networks},
	author = {Vapnik, V. N.},
	year = {1999},
	pages = {988--999},
}

@misc{kindap_non-gaussian_2022,
	title = {Non-{Gaussian} {Process} {Regression}},
	url = {http://arxiv.org/abs/2209.03117},
	doi = {10.48550/arXiv.2209.03117},
	urldate = {2025-10-27},
	publisher = {arXiv},
	author = {Kındap, Yaman and Godsill, Simon},
	month = sep,
	year = {2022},
	note = {arXiv:2209.03117 [stat]},
}

@article{paciorek_spatial_2006,
	title = {Spatial {Modelling} {Using} a {New} {Class} of {Nonstationary} {Covariance} {Functions}},
	volume = {17},
	issn = {1180-4009},
	url = {https://pmc.ncbi.nlm.nih.gov/articles/PMC2157553/},
	doi = {10.1002/env.785},
	number = {5},
	urldate = {2025-10-27},
	journal = {Environmetrics},
	author = {Paciorek, Christopher J. and Schervish, Mark J.},
	year = {2006},
	pages = {483--506},
}

@inproceedings{lalchand_sparse_2022,
	address = {Red Hook, NY, USA},
	series = {{NIPS} '22},
	title = {Sparse {Gaussian} process hyperparameters: optimize or integrate?},
	isbn = {978-1-7138-7108-8},
	shorttitle = {Sparse {Gaussian} process hyperparameters},
	urldate = {2025-10-27},
	booktitle = {Proceedings of the 36th {International} {Conference} on {Neural} {Information} {Processing} {Systems}},
	publisher = {Curran Associates Inc.},
	author = {Lalchand, Vidhi and Bruinsma, Wessel P. and Burt, David R. and Rasmussen, Carl E.},
	month = nov,
	year = {2022},
	pages = {16612--16623},
}

@book{stein_interpolation_1999,
	address = {New York, NY},
	series = {Springer {Series} in {Statistics}},
	title = {Interpolation of {Spatial} {Data}},
	copyright = {http://www.springer.com/tdm},
	isbn = {978-1-4612-7166-6 978-1-4612-1494-6},
	url = {http://link.springer.com/10.1007/978-1-4612-1494-6},
	doi = {10.1007/978-1-4612-1494-6},
	urldate = {2025-10-27},
	publisher = {Springer},
	author = {Stein, Michael L.},
	year = {1999},
}

@article{tripathy_gaussian_2016,
	title = {Gaussian processes with built-in dimensionality reduction: {Applications} to high-dimensional uncertainty propagation},
	volume = {321},
	issn = {0021-9991},
	shorttitle = {Gaussian processes with built-in dimensionality reduction},
	url = {https://www.sciencedirect.com/science/article/pii/S002199911630184X},
	doi = {10.1016/j.jcp.2016.05.039},
	urldate = {2025-10-27},
	journal = {Journal of Computational Physics},
	author = {Tripathy, Rohit and Bilionis, Ilias and Gonzalez, Marcial},
	month = sep,
	year = {2016},
	pages = {191--223},
}

@inproceedings{snelson_variable_2006,
	address = {Arlington, Virginia, USA},
	series = {{UAI}'06},
	title = {Variable noise and dimensionality reduction for sparse {Gaussian} processes},
	isbn = {978-0-9749039-2-7},
	urldate = {2025-10-27},
	booktitle = {Proceedings of the {Twenty}-{Second} {Conference} on {Uncertainty} in {Artificial} {Intelligence}},
	publisher = {AUAI Press},
	author = {Snelson, Edward and Ghahramani, Zoubin},
	month = jul,
	year = {2006},
	pages = {461--468},
}

@article{zhou_kernel_2020,
	title = {Kernel principal component analysis-based {Gaussian} process regression modelling for high-dimensional reliability analysis},
	volume = {241},
	issn = {0045-7949},
	url = {https://www.sciencedirect.com/science/article/pii/S0045794920301619},
	doi = {10.1016/j.compstruc.2020.106358},
	urldate = {2025-10-27},
	journal = {Computers \& Structures},
	author = {Zhou, Tong and Peng, Yongbo},
	month = dec,
	year = {2020},
	pages = {106358},
}

@misc{xu_standard_2025,
	title = {Standard {Gaussian} {Process} is {All} {You} {Need} for {High}-{Dimensional} {Bayesian} {Optimization}},
	url = {http://arxiv.org/abs/2402.02746},
	doi = {10.48550/arXiv.2402.02746},
	urldate = {2025-10-27},
	publisher = {arXiv},
	author = {Xu, Zhitong and Wang, Haitao and Phillips, Jeff M. and Zhe, Shandian},
	month = mar,
	year = {2025},
	note = {arXiv:2402.02746 [cs]},
}

@article{kennedy_bayesian_2001,
	title = {Bayesian {Calibration} of {Computer} {Models}},
	volume = {63},
	issn = {1369-7412},
	url = {https://doi.org/10.1111/1467-9868.00294},
	doi = {10.1111/1467-9868.00294},
	number = {3},
	urldate = {2025-10-27},
	journal = {Journal of the Royal Statistical Society Series B: Statistical Methodology},
	author = {Kennedy, Marc C. and O'Hagan, Anthony},
	month = sep,
	year = {2001},
	pages = {425--464},
}

@article{wang_recent_2022,
	title = {Recent {Advances} in {Surrogate} {Modeling} {Methods} for {Uncertainty} {Quantification} and {Propagation}},
	volume = {14},
	copyright = {http://creativecommons.org/licenses/by/3.0/},
	issn = {2073-8994},
	url = {https://www.mdpi.com/2073-8994/14/6/1219},
	doi = {10.3390/sym14061219},
	language = {en},
	number = {6},
	urldate = {2025-10-27},
	journal = {Symmetry},
	publisher = {Multidisciplinary Digital Publishing Institute},
	author = {Wang, Chong and Qiang, Xin and Xu, Menghui and Wu, Tao},
	month = jun,
	year = {2022},
	pages = {1219},
}

@article{sacks_design_1989,
	title = {Design and {Analysis} of {Computer} {Experiments}},
	volume = {4},
	issn = {0883-4237, 2168-8745},
	url = {https://projecteuclid.org/journals/statistical-science/volume-4/issue-4/Design-and-Analysis-of-Computer-Experiments/10.1214/ss/1177012413.full},
	doi = {10.1214/ss/1177012413},
	number = {4},
	urldate = {2025-10-27},
	journal = {Statistical Science},
	publisher = {Institute of Mathematical Statistics},
	author = {Sacks, Jerome and Welch, William J. and Mitchell, Toby J. and Wynn, Henry P.},
	month = nov,
	year = {1989},
	pages = {409--423},
}

@article{qian_model_2022,
	title = {Model {Reduction} of {Linear} {Dynamical} {Systems} via {Balancing} for {Bayesian} {Inference}},
	volume = {91},
	issn = {1573-7691},
	url = {https://doi.org/10.1007/s10915-022-01798-8},
	doi = {10.1007/s10915-022-01798-8},
	language = {en},
	number = {1},
	urldate = {2025-10-27},
	journal = {Journal of Scientific Computing},
	author = {Qian, Elizabeth and Tabeart, Jemima M. and Beattie, Christopher and Gugercin, Serkan and Jiang, Jiahua and Kramer, Peter R. and Narayan, Akil},
	month = mar,
	year = {2022},
	pages = {29},
}

@article{cui_datadriven_2015,
	title = {Data‐driven model reduction for the {Bayesian} solution of inverse problems},
	volume = {102},
	copyright = {http://onlinelibrary.wiley.com/termsAndConditions\#vor},
	issn = {0029-5981, 1097-0207},
	url = {https://onlinelibrary.wiley.com/doi/10.1002/nme.4748},
	doi = {10.1002/nme.4748},
	language = {en},
	number = {5},
	urldate = {2025-10-27},
	journal = {International Journal for Numerical Methods in Engineering},
	author = {Cui, Tiangang and Marzouk, Youssef M. and Willcox, Karen E.},
	month = may,
	year = {2015},
	pages = {966--990},
}

@article{schaden_multilevel_2020,
	title = {On {Multilevel} {Best} {Linear} {Unbiased} {Estimators}},
	volume = {8},
	url = {https://epubs.siam.org/doi/10.1137/19M1263534},
	doi = {10.1137/19M1263534},
	number = {2},
	urldate = {2025-10-21},
	journal = {SIAM/ASA Journal on Uncertainty Quantification},
	publisher = {Society for Industrial and Applied Mathematics},
	author = {Schaden, Daniel and Ullmann, Elisabeth},
	month = jan,
	year = {2020},
	pages = {601--635},
}

@article{gorodetsky_generalized_2020,
	title = {A generalized approximate control variate framework for multifidelity uncertainty quantification},
	volume = {408},
	issn = {0021-9991},
	url = {https://www.sciencedirect.com/science/article/pii/S0021999120300310},
	doi = {10.1016/j.jcp.2020.109257},
	urldate = {2025-10-17},
	journal = {Journal of Computational Physics},
	author = {Gorodetsky, Alex A. and Geraci, Gianluca and Eldred, Michael S. and Jakeman, John D.},
	month = may,
	year = {2020},
	pages = {109257},
}

@article{peherstorfer_survey_2018,
	title = {Survey of {Multifidelity} {Methods} in {Uncertainty} {Propagation}, {Inference}, and {Optimization}},
	volume = {60},
	issn = {0036-1445},
	url = {https://epubs.siam.org/doi/10.1137/16M1082469},
	doi = {10.1137/16M1082469},
	number = {3},
	urldate = {2025-10-17},
	journal = {SIAM Review},
	publisher = {Society for Industrial and Applied Mathematics},
	author = {Peherstorfer, Benjamin and Willcox, Karen and Gunzburger, Max},
	month = jan,
	year = {2018},
	pages = {550--591},
}

@inproceedings{bonilla_multi-task_2007,
	title = {Multi-task {Gaussian} {Process} {Prediction}},
	volume = {20},
	url = {https://papers.nips.cc/paper_files/paper/2007/hash/66368270ffd51418ec58bd793f2d9b1b-Abstract.html},
	urldate = {2025-10-15},
	booktitle = {Advances in {Neural} {Information} {Processing} {Systems}},
	publisher = {Curran Associates, Inc.},
	author = {Bonilla, Edwin V and Chai, Kian and Williams, Christopher},
	year = {2007},
}

@article{alvarez_kernels_2012,
	title = {Kernels for {Vector}-{Valued} {Functions}: {A} {Review}},
	volume = {4},
	issn = {1935-8237, 1935-8245},
	shorttitle = {Kernels for {Vector}-{Valued} {Functions}},
	url = {https://www.nowpublishers.com/article/Details/MAL-036},
	doi = {10.1561/2200000036},
	language = {English},
	number = {3},
	urldate = {2025-10-15},
	journal = {Foundations and Trends® in Machine Learning},
	publisher = {Now Publishers, Inc.},
	author = {Álvarez, Mauricio A. and Rosasco, Lorenzo and Lawrence, Neil D.},
	month = jun,
	year = {2012},
	pages = {195--266},
}

@article{brevault_overview_2020,
	title = {Overview of {Gaussian} process based multi-fidelity techniques with variable relationship between fidelities, application to aerospace systems},
	volume = {107},
	issn = {1270-9638},
	url = {https://www.sciencedirect.com/science/article/pii/S127096382031021X},
	doi = {10.1016/j.ast.2020.106339},
	urldate = {2025-10-15},
	journal = {Aerospace Science and Technology},
	author = {Brevault, Loïc and Balesdent, Mathieu and Hebbal, Ali},
	month = dec,
	year = {2020},
	pages = {106339},
}

@article{qian_multifidelity_2025,
	title = {Multifidelity linear regression for scientific machine learning from scarce data},
	volume = {7},
	copyright = {http://creativecommons.org/licenses/by/3.0/},
	url = {https://www.aimsciences.org/en/article/doi/10.3934/fods.2024049},
	doi = {10.3934/fods.2024049},
	language = {en},
	number = {1},
	urldate = {2025-10-15},
	journal = {Foundations of Data Science},
	publisher = {Foundations of Data Science},
	author = {Qian, Elizabeth and Kang, Dayoung and Sella, Vignesh and Chaudhuri, Anirban},
	month = mar,
	year = {2025},
	pages = {271--297},
}

@article{goulard_linear_1992,
	title = {Linear coregionalization model: {Tools} for estimation and choice of cross-variogram matrix},
	volume = {24},
	issn = {1573-8868},
	shorttitle = {Linear coregionalization model},
	url = {https://doi.org/10.1007/BF00893750},
	doi = {10.1007/BF00893750},
	language = {en},
	number = {3},
	urldate = {2025-09-09},
	journal = {Mathematical Geology},
	author = {Goulard, M. and Voltz, M.},
	month = apr,
	year = {1992},
	pages = {269--286},
}

@misc{pfortner_physics-informed_2024,
	title = {Physics-{Informed} {Gaussian} {Process} {Regression} {Generalizes} {Linear} {PDE} {Solvers}},
	url = {http://arxiv.org/abs/2212.12474},
	doi = {10.48550/arXiv.2212.12474},
	urldate = {2025-08-24},
	publisher = {arXiv},
	author = {Pförtner, Marvin and Steinwart, Ingo and Hennig, Philipp and Wenger, Jonathan},
	month = apr,
	year = {2024},
	note = {arXiv:2212.12474 [cs]},
}

@incollection{peterson_overview_2024,
	series = {{AIAA} {SciTech} {Forum}},
	title = {Overview of a {New} {Project} for {CFD} {Validation} of {Supersonic} {Mixing} and {Combustion}},
	url = {https://arc.aiaa.org/doi/10.2514/6.2024-1189},
	doi = {10.2514/6.2024-1189},
	urldate = {2025-08-13},
	booktitle = {{AIAA} {SCITECH} 2024 {Forum}},
	publisher = {American Institute of Aeronautics and Astronautics},
	author = {Peterson, David M. and Hammack, Stephen D. and Carter, Campbell D. and DeBardelaben, Cannon and Ochs, Bradley and Baurle, Robert A.},
	month = jan,
	year = {2024},
}

@article{koziel_multi-level_2013,
	series = {2013 {International} {Conference} on {Computational} {Science}},
	title = {Multi-level {CFD}-based {Airfoil} {Shape} {Optimization} {With} {Automated} {Low}-fidelity {Model} {Selection}},
	volume = {18},
	issn = {1877-0509},
	url = {https://www.sciencedirect.com/science/article/pii/S1877050913003979},
	doi = {10.1016/j.procs.2013.05.254},
	urldate = {2025-08-10},
	journal = {Procedia Computer Science},
	author = {Koziel, Slawomir and Leifsson, Leifur},
	month = jan,
	year = {2013},
	pages = {889--898},
}

@article{pitz_chapter_2020,
	title = {Chapter 14 {Molecular} {Tagging} {Velocimetry} in {Gases}},
	language = {en},
	author = {Pitz, Robert W and Danehy, Paul M and Steinberg, Edited Adam and Roy, Sukesh},
	month = jul,
	year = {2020},
}

@article{wang_experimental_2022,
	title = {Experimental and {Numerical} {Study} of the {Laminar} {Burning} {Velocity} and {Pollutant} {Emissions} of the {Mixture} {Gas} of {Methane} and {Carbon} {Dioxide}},
	volume = {19},
	issn = {1661-7827},
	url = {https://www.ncbi.nlm.nih.gov/pmc/articles/PMC8871781/},
	doi = {10.3390/ijerph19042078},
	number = {4},
	urldate = {2025-07-22},
	journal = {International Journal of Environmental Research and Public Health},
	author = {Wang, Yalin and Wang, Yu and Zhang, Xueqian and Zhou, Guoping and Yan, Beibei and Bastiaans, Rob J. M.},
	month = feb,
	year = {2022},
	pages = {2078},
}

@article{adusumilli_laminar_2021,
	title = {Laminar flame speed measurements of ethylene at high preheat temperatures and for diluted oxidizers},
	volume = {233},
	issn = {0010-2180},
	url = {https://www.sciencedirect.com/science/article/pii/S0010218021003072},
	doi = {10.1016/j.combustflame.2021.111564},
	urldate = {2025-07-22},
	journal = {Combustion and Flame},
	author = {Adusumilli, Sampath and Seitzman, Jerry},
	month = nov,
	year = {2021},
	pages = {111564},
}

@article{he_active_2024,
	title = {Active learning inspired multi-fidelity probabilistic modelling of geomaterial property},
	volume = {432},
	issn = {0045-7825},
	url = {https://www.sciencedirect.com/science/article/pii/S0045782524006285},
	doi = {10.1016/j.cma.2024.117373},
	urldate = {2025-07-09},
	journal = {Computer Methods in Applied Mechanics and Engineering},
	author = {He, Geng-Fu and Zhang, Pin and Yin, Zhen-Yu},
	month = dec,
	year = {2024},
	pages = {117373},
}

@article{matheron_principles_1963,
	title = {Principles of geostatistics},
	volume = {58},
	issn = {0361-0128},
	url = {https://doi.org/10.2113/gsecongeo.58.8.1246},
	doi = {10.2113/gsecongeo.58.8.1246},
	number = {8},
	urldate = {2025-06-23},
	journal = {Economic Geology},
	author = {Matheron, Georges},
	month = dec,
	year = {1963},
	pages = {1246--1266},
}

@article{krige_statistical_1951,
	title = {A statistical approach to some basic mine valuation problems on the {Witwatersrand}},
	volume = {52},
	url = {https://journals.co.za/doi/10.10520/AJA0038223X_4792},
	doi = {10.10520/AJA0038223X_4792},
	number = {6},
	urldate = {2025-06-23},
	journal = {Journal of the Southern African Institute of Mining and Metallurgy},
	publisher = {Southern African Institute of Mining and Metallurgy},
	author = {Krige, D. G.},
	month = dec,
	year = {1951},
	pages = {119--139},
}

@inproceedings{rahimi_random_2007,
	title = {Random {Features} for {Large}-{Scale} {Kernel} {Machines}},
	volume = {20},
	url = {https://papers.nips.cc/paper_files/paper/2007/hash/013a006f03dbc5392effeb8f18fda755-Abstract.html},
	urldate = {2025-06-23},
	booktitle = {Advances in {Neural} {Information} {Processing} {Systems}},
	publisher = {Curran Associates, Inc.},
	author = {Rahimi, Ali and Recht, Benjamin},
	year = {2007},
}

@inproceedings{titsias_variational_2009,
	title = {Variational {Learning} of {Inducing} {Variables} in {Sparse} {Gaussian} {Processes}},
	issn = {1938-7228},
	url = {https://proceedings.mlr.press/v5/titsias09a.html},
	language = {en},
	urldate = {2025-06-23},
	booktitle = {Proceedings of the {Twelfth} {International} {Conference} on {Artificial} {Intelligence} and {Statistics}},
	publisher = {PMLR},
	author = {Titsias, Michalis},
	month = apr,
	year = {2009},
	pages = {567--574},
}

@inproceedings{damianou_deep_2013,
	title = {Deep {Gaussian} {Processes}},
	issn = {1938-7228},
	url = {https://proceedings.mlr.press/v31/damianou13a.html},
	language = {en},
	urldate = {2025-06-20},
	booktitle = {Proceedings of the {Sixteenth} {International} {Conference} on {Artificial} {Intelligence} and {Statistics}},
	publisher = {PMLR},
	author = {Damianou, Andreas and Lawrence, Neil D.},
	month = apr,
	year = {2013},
	pages = {207--215},
}

@article{peherstorfer_optimal_2016,
	title = {Optimal {Model} {Management} for {Multifidelity} {Monte} {Carlo} {Estimation}},
	volume = {38},
	issn = {1064-8275, 1095-7197},
	url = {http://epubs.siam.org/doi/10.1137/15M1046472},
	doi = {10.1137/15M1046472},
	language = {en},
	number = {5},
	urldate = {2025-06-10},
	journal = {SIAM Journal on Scientific Computing},
	author = {Peherstorfer, Benjamin and Willcox, Karen and Gunzburger, Max},
	month = jan,
	year = {2016},
	pages = {A3163--A3194},
}

@article{gorodetsky_mfnets_2020,
	title = {{MFNets}: {MULTI}-{FIDELITY} {DATA}-{DRIVEN} {NETWORKS} {FOR} {BAYESIAN} {LEARNING} {AND} {PREDICTION}},
	volume = {10},
	issn = {2152-5080},
	shorttitle = {{MFNets}},
	url = {http://www.dl.begellhouse.com/journals/52034eb04b657aea,3673619972b2eee6,3606e2c805e529ee.html},
	doi = {10.1615/Int.J.UncertaintyQuantification.2020032978},
	language = {en},
	number = {6},
	urldate = {2025-06-07},
	journal = {International Journal for Uncertainty Quantification},
	author = {Gorodetsky, Alex A. and Jakeman, John D. and Geraci, Gianluca and Eldred, Michael S.},
	year = {2020},
	pages = {595--622},
}

@phdthesis{rasmussen_evaluation_1997,
	address = {CAN},
	type = {phd},
	title = {Evaluation of gaussian processes and other methods for non-linear regression},
	school = {University of Toronto},
	author = {Rasmussen, Carl Edward},
	year = {1997},
	note = {AAINQ28300
ISBN-10: 0612283003},
}

@inproceedings{wilson_kernel_2015,
	title = {Kernel {Interpolation} for {Scalable} {Structured} {Gaussian} {Processes} ({KISS}-{GP})},
	issn = {1938-7228},
	url = {https://proceedings.mlr.press/v37/wilson15.html},
	language = {en},
	urldate = {2025-04-30},
	booktitle = {Proceedings of the 32nd {International} {Conference} on {Machine} {Learning}},
	publisher = {PMLR},
	author = {Wilson, Andrew and Nickisch, Hannes},
	month = jun,
	year = {2015},
	pages = {1775--1784},
}

@misc{wilson_deep_2015,
	title = {Deep {Kernel} {Learning}},
	url = {http://arxiv.org/abs/1511.02222},
	doi = {10.48550/arXiv.1511.02222},
	urldate = {2025-04-16},
	publisher = {arXiv},
	author = {Wilson, Andrew Gordon and Hu, Zhiting and Salakhutdinov, Ruslan and Xing, Eric P.},
	month = nov,
	year = {2015},
	note = {arXiv:1511.02222 [cs]},
}

@article{fine_efficient_2001,
	title = {Efficient {SVM} {Training} {Using} {Low}-{Rank} {Kernel} {Representations}},
	volume = {2},
	issn = {ISSN 1533-7928},
	url = {https://www.jmlr.org/papers/v2/fine01a.html},
	number = {Dec},
	urldate = {2025-03-13},
	journal = {Journal of Machine Learning Research},
	author = {Fine, Shai and Scheinberg, Katya},
	year = {2001},
	pages = {243--264},
}

@inproceedings{williams_using_2000,
	title = {Using the {Nyström} {Method} to {Speed} {Up} {Kernel} {Machines}},
	volume = {13},
	url = {https://papers.nips.cc/paper_files/paper/2000/hash/19de10adbaa1b2ee13f77f679fa1483a-Abstract.html},
	urldate = {2025-03-13},
	booktitle = {Advances in {Neural} {Information} {Processing} {Systems}},
	publisher = {MIT Press},
	author = {Williams, Christopher and Seeger, Matthias},
	year = {2000},
}

@inproceedings{smola_sparse_2000,
	address = {San Francisco, CA, USA},
	series = {{ICML} '00},
	title = {Sparse {Greedy} {Matrix} {Approximation} for {Machine} {Learning}},
	isbn = {978-1-55860-707-1},
	urldate = {2025-03-13},
	booktitle = {Proceedings of the {Seventeenth} {International} {Conference} on {Machine} {Learning}},
	publisher = {Morgan Kaufmann Publishers Inc.},
	author = {Smola, Alex J. and Schökopf, Bernhard},
	month = jun,
	year = {2000},
	pages = {911--918},
}

@article{fernandez-godino_review_2023,
	title = {Review of multi-fidelity models},
	volume = {1},
	issn = {2837-1739},
	url = {http://arxiv.org/abs/1609.07196},
	doi = {10.3934/acse.2023015},
	number = {4},
	urldate = {2024-12-31},
	journal = {Advances in Computational Science and Engineering},
	author = {Fernández-Godino, M. Giselle},
	year = {2023},
	note = {arXiv:1609.07196 [stat]},
	pages = {351--400},
}

@misc{raissi_deep_2016,
	title = {Deep {Multi}-fidelity {Gaussian} {Processes}},
	url = {http://arxiv.org/abs/1604.07484},
	doi = {10.48550/arXiv.1604.07484},
	urldate = {2024-12-31},
	publisher = {arXiv},
	author = {Raissi, Maziar and Karniadakis, George},
	month = apr,
	year = {2016},
	note = {arXiv:1604.07484 [cs]},
}

@inproceedings{dai_scalable_2014,
	title = {Scalable {Kernel} {Methods} via {Doubly} {Stochastic} {Gradients}},
	volume = {27},
	url = {https://proceedings.neurips.cc/paper_files/paper/2014/hash/95d309f0b035d97f69902e7972c2b2e6-Abstract.html},
	urldate = {2024-12-20},
	booktitle = {Advances in {Neural} {Information} {Processing} {Systems}},
	publisher = {Curran Associates, Inc.},
	author = {Dai, Bo and Xie, Bo and He, Niao and Liang, Yingyu and Raj, Anant and Balcan, Maria-Florina F and Song, Le},
	year = {2014},
}

@article{cutajar_deep_2019,
	title = {Deep {Gaussian} {Processes} for {Multi}-fidelity {Modeling}},
	url = {https://www.semanticscholar.org/paper/Deep-Gaussian-Processes-for-Multi-fidelity-Modeling-Cutajar-Pullin/484ddd91f273edf4201e3d001d0b29a52fa27eac},
	urldate = {2024-12-08},
	journal = {ArXiv},
	author = {Cutajar, Kurt and Pullin, Mark and Damianou, Andreas C. and Lawrence, Neil D. and González, Javier I.},
	month = mar,
	year = {2019},
}

@article{kennedy_predicting_2000,
	title = {Predicting the {Output} from a {Complex} {Computer} {Code} {When} {Fast} {Approximations} {Are} {Available}},
	volume = {87},
	issn = {0006-3444},
	url = {https://www.jstor.org/stable/2673557},
	number = {1},
	urldate = {2024-11-23},
	journal = {Biometrika},
	publisher = {[Oxford University Press, Biometrika Trust]},
	author = {Kennedy, M. C. and O'Hagan, A.},
	year = {2000},
	pages = {1--13},
}

@article{gratiet_recursive_2014,
	title = {{RECURSIVE} {CO}-{KRIGING} {MODEL} {FOR} {DESIGN} {OF} {COMPUTER} {EXPERIMENTS} {WITH} {MULTIPLE} {LEVELS} {OF} {FIDELITY}},
	volume = {4},
	issn = {2152-5080, 2152-5099},
	url = {https://www.dl.begellhouse.com/journals/52034eb04b657aea,2f7b99cc281f2702,4c83626c5e64207a.html},
	doi = {10.1615/Int.J.UncertaintyQuantification.2014006914},
	language = {English},
	number = {5},
	urldate = {2024-11-18},
	journal = {International Journal for Uncertainty Quantification},
	publisher = {Begel House Inc.},
	author = {Gratiet, Loic Le and Garnier, Josselin},
	year = {2014},
}

@article{zhou_generalized_2020,
	title = {A generalized hierarchical co-{Kriging} model for multi-fidelity data fusion},
	volume = {62},
	issn = {1615-1488},
	url = {https://doi.org/10.1007/s00158-020-02583-7},
	doi = {10.1007/s00158-020-02583-7},
	language = {en},
	number = {4},
	urldate = {2024-11-18},
	journal = {Structural and Multidisciplinary Optimization},
	author = {Zhou, Qi and Wu, Yuda and Guo, Zhendong and Hu, Jiexiang and Jin, Peng},
	month = oct,
	year = {2020},
	pages = {1885--1904},
}

@article{zhou_sequential_2017,
	title = {A sequential multi-fidelity metamodeling approach for data regression},
	volume = {134},
	issn = {0950-7051},
	url = {https://www.sciencedirect.com/science/article/pii/S0950705117303556},
	doi = {10.1016/j.knosys.2017.07.033},
	urldate = {2024-11-18},
	journal = {Knowledge-Based Systems},
	author = {Zhou, Qi and Wang, Yan and Choi, Seung-Kyum and Jiang, Ping and Shao, Xinyu and Hu, Jiexiang},
	month = oct,
	year = {2017},
	pages = {199--212},
}

@article{tang_combined_2024,
	title = {A combined modeling method for complex multi-fidelity data fusion},
	volume = {5},
	issn = {2632-2153},
	url = {https://dx.doi.org/10.1088/2632-2153/ad718f},
	doi = {10.1088/2632-2153/ad718f},
	language = {en},
	number = {3},
	urldate = {2024-11-18},
	journal = {Machine Learning: Science and Technology},
	publisher = {IOP Publishing},
	author = {Tang, Lei and Liu, Feng and Wu, Anping and Li, Yubo and Jiang, Wanqiu and Wang, Qingfeng and Huang, Jun},
	month = sep,
	year = {2024},
	pages = {035071},
}

@article{zhang_multi-fidelity_2024,
	title = {{MULTI}-{FIDELITY} {MACHINE} {LEARNING} {FOR} {UNCERTAINTY} {QUANTIFICATION} {AND} {OPTIMIZATION}},
	volume = {5},
	issn = {2689-3967, 2689-3975},
	url = {https://www.dl.begellhouse.com/journals/558048804a15188a,6ea623526214d90d,07c5a8ef63c14b9f.html},
	doi = {10.1615/JMachLearnModelComput.2024055786},
	language = {English},
	number = {4},
	urldate = {2024-11-18},
	journal = {Journal of Machine Learning for Modeling and Computing},
	publisher = {Begel House Inc.},
	author = {Zhang, Ruda and Alemazkoor, Negin},
	year = {2024},
}

@article{yang_gaussian_2024,
	title = {Gaussian process fusion method for multi-fidelity data with heterogeneity distribution in aerospace vehicle flight dynamics},
	volume = {138},
	issn = {0952-1976},
	url = {https://www.sciencedirect.com/science/article/pii/S0952197624013861},
	doi = {10.1016/j.engappai.2024.109228},
	urldate = {2024-11-18},
	journal = {Engineering Applications of Artificial Intelligence},
	author = {Yang, Ben and Chen, Boyi and Liu, Yanbin and Chen, Jinbao},
	month = dec,
	year = {2024},
	pages = {109228},
}

@article{perdikaris_nonlinear_2017,
	title = {Nonlinear information fusion algorithms for data-efficient multi-fidelity modelling},
	volume = {473},
	url = {https://royalsocietypublishing.org/doi/full/10.1098/rspa.2016.0751},
	doi = {10.1098/rspa.2016.0751},
	number = {2198},
	urldate = {2024-11-18},
	journal = {Proceedings of the Royal Society A: Mathematical, Physical and Engineering Sciences},
	publisher = {Royal Society},
	author = {Perdikaris, P. and Raissi, M. and Damianou, A. and Lawrence, N. D. and Karniadakis, G. E.},
	month = feb,
	year = {2017},
	pages = {20160751},
}

@book{rasmussen_gaussian_2008,
	address = {Cambridge, Mass.},
	edition = {3. print},
	series = {Adaptive computation and machine learning},
	title = {Gaussian processes for machine learning},
	isbn = {978-0-262-18253-9},
	language = {en},
	publisher = {MIT Press},
	author = {Rasmussen, Carl Edward and Williams, Christopher K. I.},
	year = {2008},
}

@phdthesis{gratiet_multi-fidelity_2013,
	type = {phdthesis},
	title = {Multi-fidelity {Gaussian} process regression for computer experiments},
	url = {https://theses.hal.science/tel-00866770},
	language = {fr},
	urldate = {2024-10-28},
	school = {Université Paris-Diderot - Paris VII},
	author = {Gratiet, Loic Le},
	month = oct,
	year = {2013},
}

@article{xiao_extended_2018,
	title = {Extended {Co}-{Kriging} interpolation method based on multi-fidelity data},
	volume = {323},
	issn = {0096-3003},
	url = {https://www.sciencedirect.com/science/article/pii/S0096300317307646},
	doi = {10.1016/j.amc.2017.10.055},
	urldate = {2024-10-28},
	journal = {Applied Mathematics and Computation},
	author = {Xiao, Manyu and Zhang, Guohua and Breitkopf, Piotr and Villon, Pierre and Zhang, Weihong},
	month = apr,
	year = {2018},
	pages = {120--131},
}

@article{myers_matrix_1982,
	title = {Matrix formulation of co-kriging},
	volume = {14},
	issn = {1573-8868},
	url = {https://doi.org/10.1007/BF01032887},
	doi = {10.1007/BF01032887},
	language = {en},
	number = {3},
	urldate = {2024-09-10},
	journal = {Journal of the International Association for Mathematical Geology},
	author = {Myers, Donald E.},
	month = jun,
	year = {1982},
	pages = {249--257},
}

\end{document}